\documentclass[3p,12pt]{elsarticle}
\journal{International Journal of Approximate Reasoning}

\usepackage{color}
\usepackage{graphicx}
\usepackage{graphics,latexsym}
\usepackage{amsmath, amssymb, amsthm, amsfonts, mathtools}
\usepackage{hyperref}
\usepackage{arydshln}
\usepackage[table]{xcolor}
\usepackage{color}
\usepackage{url}

\newtheorem{proposition}{Proposition}
\newtheorem{example}{Example}
\newtheorem{definition}{Definition}
\newtheorem{remark}{Remark}
\newtheorem{theorem}{Theorem}[section]
\newtheorem{corollary}{Corollary}[section]
\newtheorem{assumption}{Assumption}[section]
\usepackage{helvet}         
\usepackage{courier}        
\usepackage{type1cm}     

\newcommand{\tO}{\widetilde{O}}
\newcommand{\bO}{\textbf{O}}
\newcommand{\bL}{\textbf{L}}
\newcommand{\bF}{\textbf{F}}
\newcommand{\bp}{\textbf{p}}
\newcommand{\tm}{\widetilde{m}}
\newcommand{\tL}{\widetilde{L}}
\newcommand{\sfa}{\textsf{a}}
\newcommand{\sfb}{\textsf{b}}

\newcommand{\tsfa}{\widetilde{\textsf{a}}}

\newcommand{\proj}[1]{^{\downarrow #1}}

\begin{document}

\begin{frontmatter}
\title{An Interval-Valued Utility Theory for Decision Making with Dempster-Shafer Belief Functions}

\author[utc,shu,iuf]{Thierry Den\oe ux}
\author[uks]{Prakash P.~Shenoy\corref{cor1}}

\address[utc]{Universit\'e de Technologie de Compi\`egne, CNRS, UMR 7253 Heudiasyc, Compi\`egne, France\\
\url{thierry.denoeux@utc.fr}}
\address[shu]{Shanghai University, UTSEUS, Shanghai, China}
\address[iuf]{Institut Universitaire de France, Paris, France}
\address[uks]{University of Kansas, School of Business, Lawrence, Kansas, USA\\
\url{pshenoy@ku.edu}}

\cortext[cor1]{Corresponding author.}

\begin{abstract}
The main goal of this paper is to describe an axiomatic utility theory for Dempster-Shafer belief function lotteries. The axiomatic framework used is analogous to von Neumann-Morgenstern's utility theory for probabilistic lotteries as described by Luce and Raiffa. Unlike the probabilistic case, our axiomatic framework leads to interval-valued utilities, and therefore, to a partial (incomplete) preference order on the set of all belief function lotteries. If the belief function reference lotteries we use are Bayesian belief functions, then our representation theorem coincides with Jaffray's representation theorem for his linear utility theory for belief functions. We illustrate our representation theorem using some examples discussed in the literature, and we propose a simple model for assessing utilities based on an interval-valued pessimism index representing a decision-maker's attitude to ambiguity and indeterminacy. Finally, we compare our decision theory with those proposed by Jaffray, Smets, Dubois et al., Giang and Shenoy, and Shafer.
\end{abstract}


\end{frontmatter}


\section{Introduction}
\label{sec:intro}
The main goal of this paper is to propose an axiomatic utility theory for lotteries described by belief functions in the Dempster-Shafer (D-S) theory of evidence \cite{Dempster1967, Shafer1976}. The  axiomatic theory is constructed similar to von Neumann-Morgenstern's (vN-M's) utility theory for probabilistic lotteries \cite{vonNeumannMorgenstern1947, HersteinMilnor1953, Hausner1954, LuceRaiffa1957, Jensen1967, Fishburn1982}. Unlike the probabilistic case, our axiomatic theory leads to interval-valued utilities, and therefore to a partial (incomplete) preference order on the set of all belief function lotteries. Also, we compare our decision theory to those proposed by Jaffray \cite{Jaffray1989}, Smets \cite{Smets2002}, Dubois et al. \cite{Duboisetal1999}, Giang and Shenoy \cite{GiangShenoy2005, GiangShenoy2011}, and Shafer \cite{Shafer2016}.

In the foreword to Glenn Shafer's 1976 monograph \cite{Shafer1976}, Dempster writes: ``... I believe that Bayesian inference will always be a basic tool for practical everyday statistics, if only because questions must be answered and decisions must be taken, so that a statistician must always stand ready to upgrade his vaguer forms of belief into precisely additive probabilities.'' More than 40 years after these lines were written, a lot of approaches to decision-making have been proposed (see the recent review in \cite{Denoeux2019}). However, most of these methods lack a strong theoretical basis. The most important steps toward a decision theory in the  D-S framework have been made by Jaffray \cite{Jaffray1989}, Smets \cite{Smets2002} and Shafer \cite{Shafer2016}. However, we argue that these proposals are either not sufficiently justified from the point of view of D-S theory, or not sufficiently developed for practical use. Our goal is to propose and justify a utility theory that is in line with vN-M's utility theory, but adapted to be used with lotteries whose uncertainty is described by D-S belief functions.

In essence, the D-S theory consists of representations--- basic probability assignments (also called mass functions), belief functions, plausibility functions,  etc.---together with Dempster's combination rule, and a rule for marginalizing joint belief functions. The representation part of the D-S theory is also used in various other theories of belief functions. For example, in the imprecise probability community, a belief function is viewed as the lower envelope of a convex set of probability mass functions called a credal set. Credal set semantics are also referred to in the literature as \emph{lower probability interpretation} \cite{Jaffray1991, Jaffray1994, JaffrayWakker1993}, and as \emph{generalized probability} \cite{FaginHalpern1991, HalpernFagin1992}. Using these semantics, it makes more sense to use the Fagin-Halpern combination rule \cite{FaginHalpern1991} (also proposed by de Campos et al. \cite{deCamposetal1990}), rather than Dempster's combination rule \cite{HalpernFagin1992, Shafer1990, Shafer1992}. The utility theory this article proposes is designed specifically for the D-S belief function theory, and not for the  other theories of belief functions. This suggests that Dempster's combination rule should be an integral part of our theory, a property that is not satisfied in the proposals by Jaffray and Smets.

There is a large literature on decision making with a (credal) set of probability mass functions \cite{GilboaMarinacci2016} motivated by Ellsberg's paradox \cite{Ellsberg1961}. An influential work in this area is the axiomatic framework by Gilboa-Schmeidler \cite{GilboaSchmeidler1989}, where they use Choquet integration \cite{Choquet1953, GilboaSchmeidler1994} to compute expected utility. A belief function is a special case of a Choquet capacity. Jaffray's \cite{Jaffray1989} work can also be regarded as belonging to the same line of research, although Jaffray works directly with belief functions without specifying a combination rule. A review of this literature can be found in, e.g., \cite{Gajdosetal2008}, where the authors propose a modification of the Gilboa-Schmeidler \cite{GilboaSchmeidler1989} axioms. As we said earlier, our focus here is on decision-making with D-S theory of belief functions, and not on decision-making based in belief functions with a credal set interpretation. As we will see, our interval-valued utility functions lead to intervals that are contained in the Choquet lower and upper expected utility intervals.

The remainder of this article is as follows. In Section \ref{sec:vNM}, we sketch vN-M's axiomatic utility theory for probabilistic lotteries as described by Luce and Raiffa \cite{LuceRaiffa1957}. In Section \ref{sec:dstheory}, we summarize the basic definitions in the D-S belief function theory. In Section \ref{sec:dsutility}, we describe our adaptation of vN-M's utility theory for lotteries in which uncertainty is described by D-S belief functions. Our assumptions lead to an interval-valued utility function, and consequently, to a partial (incomplete) preference order on the set of all belief function lotteries. We also describe a model for assessments of utilities. In Section \ref{sec:comparison}, we compare our utility theory with those described by  Jaffray \cite{Jaffray1989}, Smets \cite{Smets2002}, Dubois et al. \cite{Duboisetal1999}, Giang and Shenoy \cite{GiangShenoy2005, GiangShenoy2011}, and Shafer \cite{Shafer2016}. Finally, in Section \ref{sec:sandc}, we summarize and conclude.


\section{von Neumann-Morgenstern's Utility Theory}
\label{sec:vNM}
In this section, we describe vN-M's utility theory for decision under risk. Most of the material in this section is adapted from \cite{LuceRaiffa1957}. A decision problem can be seen as a situation in which a decision-maker (DM) has to choose a course of action (or \emph{act}) in some set $\bF$. An act may have different \emph{outcomes}, depending on the \emph{state of nature} $X$. Exactly one state of nature will obtain, but this state is unknown. Denoting by $\Omega_X=\{x_1,\ldots,x_n\}$ the set of states of nature and by $\bO=\{O_1, \ldots, O_r\}$ the set of outcomes\footnote{The assumption of finiteness of the sets $\Omega_X$ and $\bO$ is only for ease of exposition. It is unnecessary for the proof of the representation theorem in this section.}, 
an act can thus be formalized as a mapping $f$ from $\Omega$ to $\bO$. In this section, we assume that uncertainty about the state of nature  is described by a probability mass function (PMF) $p_X$ on $\Omega_X$. In vN-M's original exposition \cite{vonNeumannMorgenstern1947}, probabilities on $\Omega_X$ are assumed to  be objective and  to correspond to long-run frequencies. However, the line of reasoning summarized below is also valid with other interpretations of probabilities, such as additive degrees of belief, provided that probabilities are assumed to have been determined beforehand, independently of the decision problem. For instance, in the constructive approach proposed by Shafer \cite{Shafer2016}, probabilities are constructed by comparing a given problem with a scale of examples in which the truth is generated according to known chances\footnote{Savage \cite{savage51} derives both probabilities and utilities from a set of axioms. This approach will not be considered in this paper.}. 

If the DM selects act $f$, they will get  outcome $O_i$ with probability 
\begin{equation}
\label{eq:prob_lottery}
p_i = \sum_{\{x \in \Omega_X \mid f(x)=O_i\}} p_X(x).
\end{equation}
To each act $f$ thus corresponds a PMF $\bp = (p_1, \ldots, p_r)$  on $\bO$.  We call $L = [\bO, \bp]$ a \emph{probabilistic lottery}. As only one state in $\Omega_X$ will obtain, a probabilistic lottery will result in exactly one outcome $O_i$ (with probability $p_i$), and we suppose that the lottery will not be repeated. Another natural assumption is that two acts that induce the same lottery are equivalent: the problem of expressing preference between acts then boils down to expressing preference between lotteries.

We are thus concerned with a DM who has preferences on $\mathcal{L}$, the set of all probabilistic lotteries on $\bO$, and our task is to find a real-valued \emph{utility function} $u: \mathcal{L} \to \mathbb{R}$ such that  the DM strictly prefers $L$ to $L^\prime$ if and only if $u(L) > u(L^\prime)$, and the DM is indifferent between $L$ and $L^\prime$ if and only if $u(L) = u(L^\prime)$. We write $O_i \succ O_j$ if the DM strictly prefers $O_i$ to $O_j$, write $O_i \sim O_j$ if the DM is indifferent between (or equally prefers) $O_i$ and $O_j$, and write $O_i \succsim O_j$ if the DM either strictly prefers $O_i$ to $O_j$ or is indifferent between the two.

Of course, finding such a utility function is not always possible, unless  the DM's preferences satisfy some assumptions. We can then construct  a utility function that is \emph{linear} in the sense that the utility of a lottery $L = [\bO, \bp]$ is equal to its expected utility $\sum_{i=1}^r p_i\,u(O_i)$, where $O_i$ is regarded as a degenerate lottery where the only possible outcome is $O_i$ with probability 1. In the remainder of this section, we describe a set of assumptions that lead to the existence of such a linear utility function.

\begin{assumption}[Weak ordering of outcomes]
\label{a1p}
For any two outcomes $O_i$ and $O_j$, either $O_i \succsim O_j$ or $O_j \succsim O_i$. Also, if $O_i \succsim O_j$ and $O_j \succsim O_k$, then $O_i \succsim O_k$. Thus, the preference relation $\succsim$ over $\bO$ is a weak order, i.e., it is complete and transitive. 
\end{assumption}

Given Assumption \ref{a1p}, without loss of generality, let us assume that the outcomes are labelled such that 
$O_1 \succsim O_2 \succsim \cdots \succsim O_r$, and to avoid trivialities, assume that $O_1 \succ O_r$.

Suppose that $\textbf{L} = \{L^{(1)}, \ldots, L^{(s)}\}$ is a set of $s$ lotteries, where each of the $s$ lotteries $L^{j} = [\bO, \bp^{(j)}]$ are over outcomes in $\bf{O}$, with PMFs $\bp^{(j)}$ for $j = 1, \ldots, s$. Suppose ${\bf q} = (q_1, \ldots, q_s)$ is a PMF on $\textbf{L}$ such that $q_j > 0$ for $j = 1, \ldots, s$, and $\sum_{j=1}^s q_j = 1$. Then $[\textbf{L}, {\bf q}]$ is called a \emph{compound} lottery whose outcome is exactly one lottery $L^{(i)}$ (with probability $q_i$), and lottery $L^{(i)}$ will result in one outcome $O_j$ (with probability $p^{(i)}_j$). Notice that the PMF $\bp^{(i)}$ is a conditional PMF for $\bO$ in the second stage given that lottery $L^{(i)}$ is realized (with probability $q_i > 0$) in the first stage (see Figure \ref{fig:redcomplot}). We can compute the joint PMF for $(\textbf{L}, \bO)$, and then compute the marginal $\bp$ of the joint for $\bO$. The following assumption states that the resulting lottery $[\bO,\bp]$ is indifferent to the compound lottery $[\textbf{L}, {\bf q}]$.

\begin{assumption}[Reduction of compound lotteries] 
\label{a2p}
Any compound lottery $[\bf{L}$, $\bf{q}]$, where $L^{(i)} = [\bO, \bp^{(i)}]$, is indifferent to a simple $($non-compound$)$ lottery $[\bf{O}$, $\bf{p}]$, where
\begin{equation}
p_i = q_1\,p_i^{(1)} + \ldots + q_s\,p_i^{(s)}
\end{equation}
for $i = 1, \ldots, r$. PMF $(p_1, \ldots, p_r)$ is the marginal for $\bO$ of the joint PMF of $(\textbf{L},\bO)$.  
\end{assumption}

\begin{figure}
\centering
\includegraphics[width=4.5in]{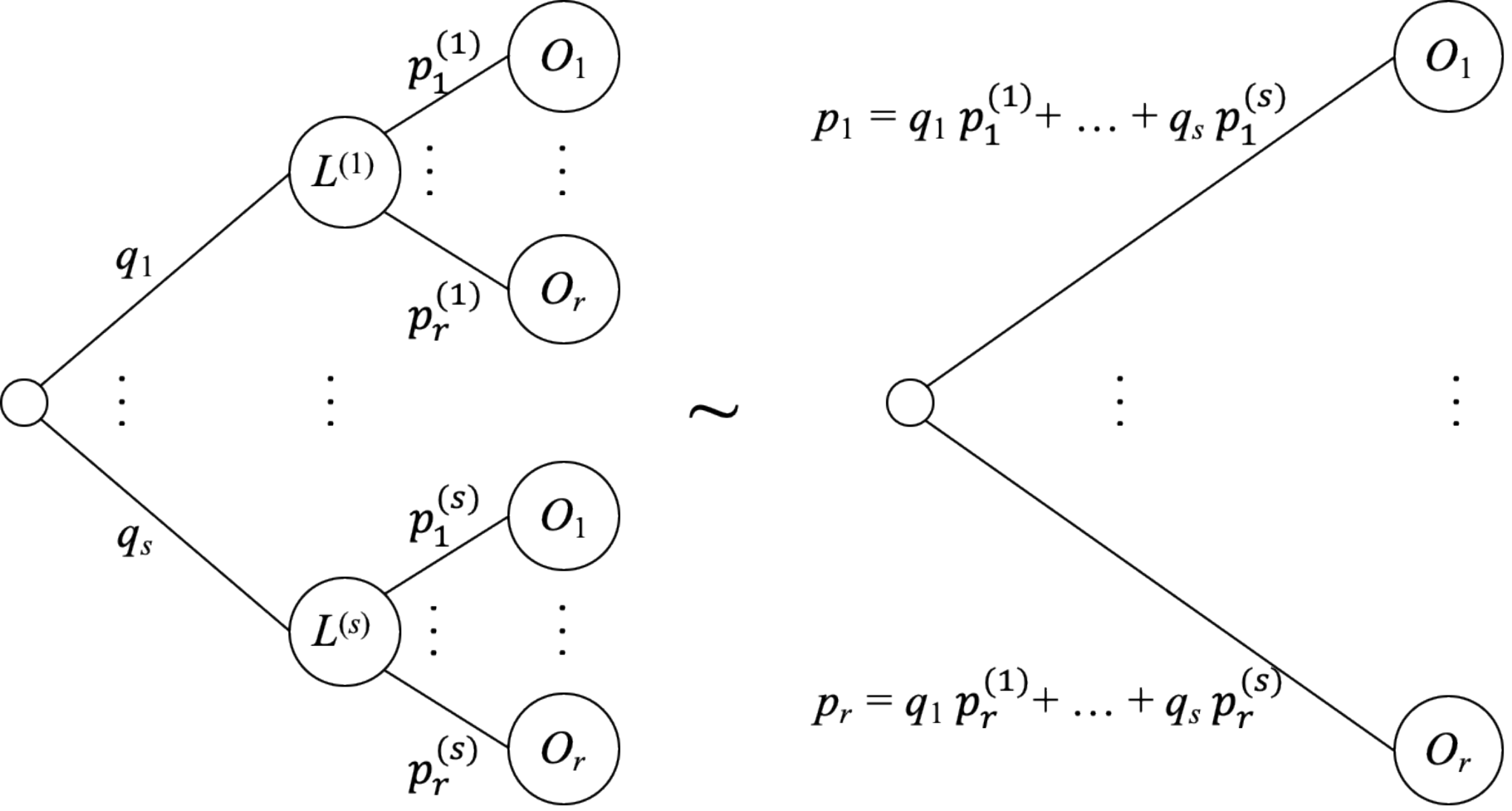}
\caption{A two-stage compound lottery reduced to an indifferent simple lottery}
\label{fig:redcomplot}
\end{figure}

A simple lottery involving only outcomes $O_1$ and $O_r$ with PMF $(u, 1-u)$, where $0 \le u \le 1$, is called a \emph{reference} lottery, 
and is denoted by $[\{O_1, O_r\},(u, 1-u)]$. Let $\bO_2$ denote the set $\{O_1, O_r\}$.

\begin{assumption}[Continuity] 
\label{a3p}
Each outcome $O_i$ is indifferent to a reference lottery \\$\widetilde{O}_i = [\bO_2, (u_i, 1-u_i)]$ for some $u_i$, 
where $0 \le u_i \le 1$, i.e., $O_i \sim \widetilde{O}_i$. 
\end{assumption}

\begin{assumption}[Weak order]
\label{a4p}
 The preference relation $\succsim$ for lotteries in $\mathcal{L}$ is a weak order, i.e., it is complete and transitive. 
\end{assumption}

Assumption \ref{a4p} generalizes Assumption \ref{a1p} for outcomes, which can be regarded as degenerate lotteries.

\begin{assumption}[Substitutability]
\label{a5p}
In any lottery $L = [\bO, \bp]$, if we substitute an outcome $O_i$ by the reference lottery 
$\widetilde{O}_i = [\bO_2, (u_i, 1-u_i)]$ that is indifferent to $O_i$, then the result is a compound lottery that is indifferent to $L$ 
$($see Figure \ref{fig:substitutability}$)$, i.e, 
\begin{equation*}
[(O_1, \ldots, O_{i-1}, O_i, O_{i+1}, \ldots, O_r), \bp] \sim [(O_1, \ldots, O_{i-1}, \widetilde{O}_i, O_{i+1}, \ldots, O_r), \bp].
\end{equation*}
\end{assumption}

\begin{figure}
\centering
\includegraphics[width=4.5in]{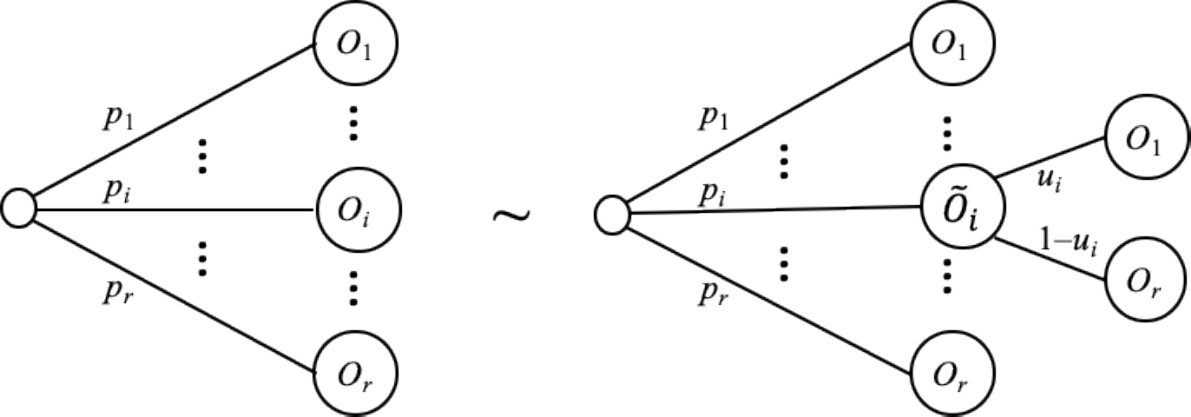}
\caption{The substitutability assumption \ref{a5p}}
\label{fig:substitutability}
\end{figure}

From Assumptions \ref{a1p}--\ref{a5p}, given any lottery $L = [\bO, \bp]$, it is possible to find a reference lottery $\tL=[\bO_2, (u, 1-u)]$ that is indifferent to $L$ (see Figure \ref{fig:reflot}). This is expressed by Theorem \ref{th:redbflottery} below.

\begin{theorem} [Reducing a lottery to an indifferent  reference lottery]
\label{th:redbflottery}
Under Assumptions \ref{a1p}-\ref{a5p}, any  lottery $L = [\bO, \bp]$   is indifferent to a reference  lottery $\tL=[\bO_2, (u, 1-u)]$ with
\begin{equation}
\label{eq:u}
u = \sum_{i=1}^r p_i\,u_i.
\end{equation}
\end{theorem}

\begin{proof}(\cite{LuceRaiffa1957})
First, we replace each $O_i$ by $\widetilde{O}_i$ for $i = 1, \ldots, r$. Assumption \ref{a3p} (continuity) states that these 
indifferent lotteries exist, and Assumption \ref{a5p} (substitutability) says that they are substitutable without changing the preference relation. 
So by using Assumption \ref{a4p} serially, $[\bO, \bp] \sim [\widetilde{\bO}, \bp]$. Now if we apply Assumption \ref{a2p} 
(reduction of compound lotteries), then $[\bO, \bp] \sim [\bO_2, (u, 1-u)]$, where $u$ is given by Eq. \eqref{eq:u}.
\end{proof}

\begin{figure}
\centering
\includegraphics[width=6in]{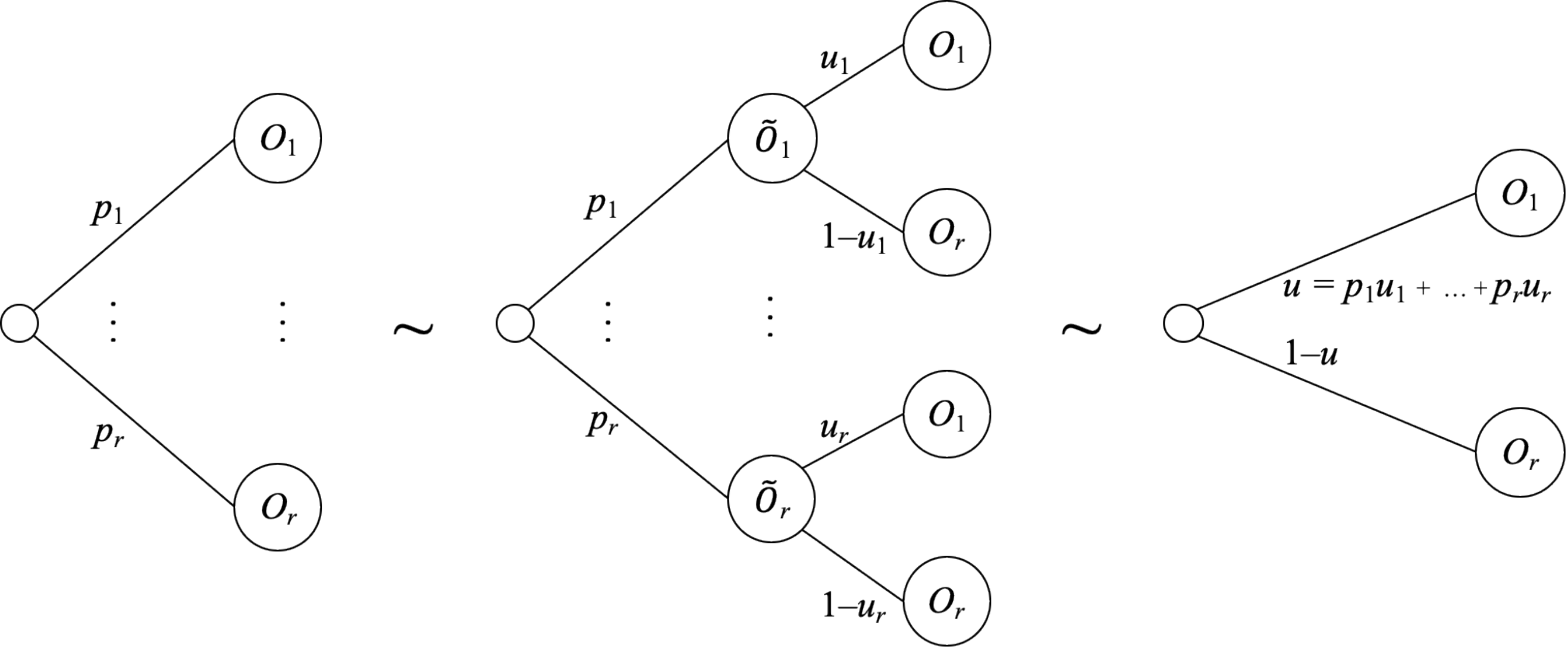}
\caption{Reducing a lottery to an indifferent compound lottery and then to an indifferent reference lottery}
\label{fig:reflot}
\end{figure}

\begin{assumption}[Monotonicity] 
\label{a6p}
A reference lottery $L = [\bO_2, (u, 1-u)]$ is preferred or indifferent to reference lottery $L^{\prime} = [\bO_2, (u^\prime, 1-u^\prime)]$ if and only if $u \ge u^\prime$. 
\end{assumption}

As $O_1 \sim \tO_1=[\bO,(u_1,1-u_1)]$ and $O_r \sim \tO_r=[\bO,(u_r,1-u_r)]$, Assumptions \ref{a4p} and \ref{a6p} imply that $u_1=1$ and $u_r=0$. Also, from $O_1 \succsim O_2 \succsim \cdots \succsim O_r$, we can deduce that $1=u_1\ge u_2\ge \cdots \ge u_r=0$.

Assumptions \ref{a1p}--\ref{a6p}  allow us to define the utility of a lottery as the probability of the best outcome $O_1$ in an indifferent reference lottery, and this utility function for lotteries on $\bO$  is linear. This is stated by the following theorem.

\begin{theorem}[\cite{LuceRaiffa1957}]
If the preference relation $\succsim$ on $\mathcal{L}$ satisfies Assumptions \ref{a1p}--\ref{a6p}, then there are numbers  $u_i$ associated with outcomes $O_i$ for $i = 1, \ldots, r$, such that for any two lotteries $L = [\bf{O}, \bf{p}]$, and $L^\prime = [\bf{O}, \bf{p}^\prime]$, $L \succsim L^\prime$ if and only if
\begin{equation}
\sum_{i=1}^r p_i\,u_i \ge \sum_{i=1}^r p_i^\prime\,u_i.
\end{equation}
Thus, we can define the utility of lottery $L = [\bO, \bp]$ as $u(L) = \sum_{i=1}^r p_i \, u_i$, where $u_i = u(O_i)$. Also, such a linear utility function is unique up to a strictly increasing affine transformation, i.e., if $u_i^\prime = a\,u_i + b$, where $a >0$ and $b$ are real constants, then $u(L) = \sum_{i=1}^r p_i\,u_i^\prime$ also qualifies as a utility function.
\end{theorem}


\section{Basic Definitions in the D-S Belief Function Theory}
\label{sec:dstheory}
In this section, we review the basic definitions in the D-S theory of belief functions. Like various uncertainty theories, D-S belief function theory includes functional representations of uncertain knowledge, and basic operations for making inferences from such knowledge. These will be recalled, respectively, in Section \ref{subsec:representations} and \ref{subsec:operations}. Conditional belief functions and the notion of conditional embedding are then introduced in Section \ref{subsec:conditional_bf}, and the semantics of belief functions in D-S theory is discussed in Section \ref{subsec:semantics}. Most of the material in this section (except Section \ref{subsec:semantics}) is taken from \cite{JirousekShenoy2020}. For further details, the reader is referred to \cite{Shafer1976} and to \cite{Denoeuxetal2019} for a recent review.

\subsection{Representations of belief functions}
\label{subsec:representations}

Belief functions can be represented in several different ways, including as basic probability assignments, plausibility functions and belief functions\footnote{
Belief functions can also be mathematically represented by a convex set of PMFs called a \emph{credal set}, 
but the semantics of such a representation are incompatible with Dempster's combination rule 
\cite{Shafer1981, Shafer1990, Shafer1992, HalpernFagin1992}. For these reasons, we skip a credal set representation of a belief function.
}. These are briefly discussed below.

\begin{definition}[Basic Probability Assignment]
Suppose $X$ is an unknown quantity $($variable$)$ with possible values $($states$)$ in a  finite set $\Omega_{X}$ called the state space of $X$. We assume that $X$ takes one and only one value in $\Omega_X$, but this value is unknown. Let $2^{\Omega_{X}}$ denote the set of all subsets of $\Omega_{X}$. A basic probability assignment $($BPA$)$ $m_{X}$ for $X$ is a function $m_{X}:2^{\Omega_{X}} \to [0, 1]$ such that
\begin{equation}
\sum_{\sfa \subseteq \Omega_{X}} m_{X}(\sfa) = 1, \quad \text{and}\quad m_X(\emptyset)=0.
\end{equation}
\label{def:bpa}
\end{definition}

The  subsets $\sfa \subseteq \Omega_{X}$ such that $m_{X}(\sfa) > 0$ are called \emph{focal sets} of $m_{X}$. 
An example of a BPA for $X$ is the \emph{vacuous} BPA for $X$, denoted by $\iota_{X}$, such that $\iota_{X}(\Omega_{X}) = 1$. 
We say that $m_{X}$ is \emph{deterministic} if $m_{X}$ has a single focal set (with mass 1). Thus, the vacuous BPA for $X$ is 
deterministic with focal set $\Omega_{X}$. If all focal sets of $m_{X}$ are singleton subsets (of $\Omega_{X}$), 
then we say that $m_X$ is \emph{Bayesian}. In this case, $m_{X}$ is equivalent to the PMF $P_{X}$ for $X$ such that $P_{X}(x) = m_{X}(\{x\})$ 
for each $x \in \Omega_{X}$.

\begin{definition}[Plausibility Function]
The information in a BPA $m_{X}$ can be represented by a corresponding plausibility function $Pl_{m_{X}}$  defined as follows:
\begin{equation}
\label{plausibility}
Pl_{m_{X}}(\sfa) = \sum_{\{\textsf{b} \subseteq \Omega_{X} \mid \textsf{b} \cap \sfa \neq \emptyset\}} m_{X}(\textsf{b}) 
\quad\text{for all $\sfa \subseteq \Omega_{X}$.}
\end{equation}
\end{definition}

For an example, suppose $\Omega_{X} = \{x, \bar{x}\}$. Then, the plausibility function $Pl_{\iota_{X}}$ corresponding to BPA $\iota_{X}$ 
is given by $Pl_{\iota_{X}}(\emptyset) = 0$, $Pl_{\iota_{X}}(\{x\}) = 1$, $Pl_{\iota_{X}}(\{\bar{x}\}) = 1$, and $Pl_{\iota_{X}}(\Omega_{X}) = 1$.

\begin{definition}[Belief Function]
The information in a BPA $m_X$ can also be represented by a corresponding belief function $Bel_{m_{X}}$ that is defined as follows:
\begin{equation}
\label{belief}
Bel_{m_{X}}(\sfa) = \sum_{\{\textsf{b} \subseteq \Omega_{X} \mid \textsf{b} \subseteq \sfa\}} m_{X}(\textsf{b})
\quad\text{for all $\sfa \subseteq \Omega_{X}$.}
\end{equation} 
\label{def:bel}
\end{definition}

For the example above with $\Omega_{X} = \{x, \bar{x}\}$, the belief function $Bel_{\iota_{X}}$ corresponding to BPA $\iota_{X}$ 
is given by $Bel_{\iota_{X}}(\emptyset) = 0$, $Bel_{\iota_{X}}(\{x\}) = 0$, $Bel_{\iota_{X}}(\{\bar{x}\}) = 0$, and $Bel_{\iota_{X}}(\Omega_{X}) = 1$. For any proposition
$\sfa \in 2^{\Omega_X}$, it is easy to see that $Bel_{m_X}(\sfa) \le Pl_{m_X}(\sfa)$.
Thus, if a DM's belief in proposition $\sfa$ is an interval, say $[p, p+q]$, where $p, q \ge 0$ and $p + q \le 1$, then such beliefs can be 
represented by a BPA $m_X$ such that $m_X(\sfa) = p$, $m_X(\Omega_X \setminus \sfa) = 1-p-q$, and $m_X(\Omega_X) = q$. 
For such a BPA, $Bel_{m_X}(\sfa) = p \le p + q = Pl_{m_X}(\sfa)$.

All three representations---BPA, belief and plausibility functions---have exactly the same information, as any one of them allows us to recover the other two \cite{Shafer1976}.

Next, we describe the two main operations for making inferences.

\subsection{Basic operations in the D-S theory}
\label{subsec:operations}

There are two main operations in the D-S theory---Dempster's combination rule and marginalization.

\paragraph{Dempster's Combination Rule}
In the D-S theory, we can combine two BPAs $m_{1}$ and $m_{2}$  representing distinct pieces of evidence by Dempster's rule \cite{Dempster1967} 
and obtain the BPA $m_{1} \oplus m_{2}$, which represents the combined evidence. 
Dempster refers to this rule as the product-intersection rule, as the product of the BPA values are assigned to the intersection of the focal sets, 
followed by normalization. Normalization consists of discarding the mass assigned to $\emptyset$, and normalizing the remaining values 
so that they add to 1. In general, Dempster's rule of combination can be used to combine two BPAs for arbitrary sets 
of variables. 

Let $\mathcal{X}$ denote a finite set of variables. The state space of $\mathcal{X}$ is $\bigtimes_{X \in \mathcal{X}} \Omega_{X}$. 
Thus, if $\mathcal{X} = \{X, Y\}$ then the state space of $\{X, Y\}$ is 
$\Omega_{X} \times \Omega_{Y}$. Projection of states simply means dropping extra coordinates; for example, if $(x,y)$ is a state of $\{X,Y\}$, 
then the projection of $(x, y)$ to $X$, denoted by $(x, y)^{\downarrow X}$, is simply $x$, which is a 
state of $X$. Projection of subsets of states is achieved by projecting every state in the subset. 
Suppose $\textsf{b} \in 2^{\Omega_{\{X, Y\}}}$. Then $\textsf{b}^{\downarrow X} = \{x \in \Omega_{X}: (x, y) \in \textsf{b}\}$.
Notice that $\textsf{b}^{\downarrow X} \in 2^{\Omega_{X}}$.

Vacuous extension of a subset of states of $\mathcal{X}_{1}$ to a subset of states of $\mathcal{X}_{2}$, 
where $\mathcal{X}_{2} \supseteq \mathcal{X}_{1}$, is a cylinder set extension, i.e., if $\mathsf{a} \in 
2^{\mathcal{X}_{1}}$, then $\sfa^{\uparrow \mathcal{X}_{2}} = \sfa \times \Omega_{\mathcal{X}_{2} \setminus \mathcal{X}_{1}}$. 
Thus, if $\mathsf{a} \in 2^{\Omega_{X}}$, then $\mathsf{a}
^{\uparrow \{X, Y\}} = \mathsf{a} \times \Omega_{Y}$.

\begin{definition}[Dempster's rule using BPAs]
Suppose $m_{1}$ and $m_{2}$ are BPAs for $\mathcal{X}_{1}$ and  $\mathcal{X}_{2}$, respectively. 
Then $m_{1} \oplus m_{2}$ is a BPA for $\mathcal{X}_{1} \cup \mathcal{X}_{2} = \mathcal{X}$, say, 
given by $(m_{1} \oplus m_{2})(\emptyset)=0$ and
\begin{equation}
\label{Dempster1}
(m_{1} \oplus m_{2})(\sfa) = K^{-1} \sum_{\{\textsf{b}_{1}, \textsf{b}_{2} \subseteq \Omega_{\mathcal{X}} \mid 
\textsf{b}_{1} \cap \textsf{b}_{2} = \sfa\}} m_{1}(\textsf{b}_{1}^{\downarrow \mathcal{X}_{1}}) \, 
m_{2}(\textsf{b}_{2}^{\downarrow \mathcal{X}_{2}}),
\end{equation}
for all $\sfa\subseteq \Omega_{\mathcal{X}}$, where $K$ is a normalization constant given by
\begin{equation}
\label{normconstant}
K = 1 - \sum_{\{\textsf{b}_{1}, \textsf{b}_{2} \subseteq \Omega_{\mathcal{X}} \mid 
\textsf{b}_{1} \cap \textsf{b}_{2} = \emptyset\}} 
m_{1}(\textsf{b}_{1}^{\downarrow \mathcal{X}_{1}}) \, m_{2}(\textsf{b}_{2}
^{\downarrow \mathcal{X}_{2}}).
\end{equation}
\label{def:dempster1}
\end{definition}
The definition of Dempster's rule assumes that the normalization constant $K$ is non-zero. 
If $K = 0$, then the two BPAs $m_{1}$ and $m_{2}$ are said to be in \emph{total conflict} and cannot be combined. 
If $K = 1$, we say $m_{1}$ and $m_{2}$ are \emph{non-conflicting}.

\paragraph{Marginalization} Marginalization in D-S theory is addition of values of BPAs.
\begin{definition}[Marginalization]
Suppose $m$ is a BPA for $\mathcal{X}$. Then, the marginal of $m$ for $\mathcal{X}_{1}$, 
where $\mathcal{X}_{1} \subset \mathcal{X}$, denoted by $m^{\downarrow \mathcal{X}_{1}}$, is a BPA for $
\mathcal{X}_{1}$ such that for each $\sfa \subseteq \Omega_{\mathcal{X}_{1}}$,
\begin{equation}
m^{\downarrow \mathcal{X}_{1}}(\sfa) = \sum_{\{\textsf{b} \subseteq  \Omega_{\mathcal{X}} \mid \textsf{b}\proj{\mathcal{X}_{1}} = 
\,\sfa\}} m(\textsf{b}).
\label{eq:marg}
\end{equation}
\label{def:marg}
\end{definition}

\subsection{Conditional belief functions}
\label{subsec:conditional_bf}
In probability theory, it is common to construct joint PMFs for a set of discrete variables by using conditional 
probability distributions. For example, we can construct joint PMF for $(X, Y)$ 
by first assessing PMF $P_{X}$ of $X$, and conditional PMFs $P_{Y|x}$ for each $x \in \Omega_{X}$ such that $P_{X}(x) > 0$. 
The pointwise multiplication of $P_{Y|x}$ for all $x \in \Omega_{X}$ is called a 
CPT, and denoted by $P_{Y|X}$. Then, $P_{X, Y} = P_{X} \otimes P_{Y|X}$. We can construct joint BPA for $\{X, Y\}$ in a similar manner.

Suppose that there is a BPA for $Y$ expressing our belief about $Y$ 
if we know that $X=x$, and denote it by $m_{Y|x}$. Notice that $m_{Y|x}: 2^{\Omega_{Y}} \to [0, 1]$  is such that 
$\sum_{\textsf{b} \in 2^{\Omega_{Y}}} m_{Y|x}(\textsf{b}) = 1$. We can embed this conditional BPA for $Y$ into a  BPA for $\{X, Y\}$, 
which is denoted by $m_{x, Y}$, such that the following three conditions hold. First, $m_{x, Y}$ tells us nothing about $X$, i.e., 
$m_{x, Y}^{\downarrow X}(\Omega_{X}) = 1$. Second, $m_{x, Y}$ tells us nothing about $Y$, i.e., 
$m_{x, Y}^{\downarrow Y}(\Omega_{Y}) = 1$. Third, if we combine $m_{x, Y}$ with the deterministic BPA 
$m_{X=x}$ for $X$ such $m_{X=x}(\{x\}) = 1$ using Dempster's rule, and marginalize the result to $Y$ we obtain $m_{Y|x}$, i.e., 
$(m_{x, Y} \oplus m_{X=x} )^{\downarrow Y}  = m_{Y|x}$. The least committed way to obtain such an embedding,  called \emph{conditional embedding}, was derived by Smets \cite{Smets1978,Smets1993} 
(see also \cite{Shafer1982}). It consists of taking each focal set $\textsf{b} \in 2^{\Omega_{Y}}$ of 
$m_{Y|x}$, and converting it to a corresponding focal set of  $m_{x, Y}$ (with the same mass) as follows: 
$(\{x\} \times \textsf{b}) \cup (\overline{\{x\}} \times \Omega_{Y})$, where $\overline{\{x\}}$ denotes the complement of $\{x\}$ in $\Omega_X$. It is easy to confirm that this method of embedding 
satisfies the three conditions mentioned above, and  $m_{x, Y}$ is the least committed (minimally informative) BPA verifying this property.

\begin{example}[Conditional embedding]
\label{ex:condembed}
Consider discrete variables $X$ and $Y$, with $\Omega_{X} = \{x, \bar{x}\}$ and $\Omega_{Y} = \{y, \bar{y}\}$. 
Suppose that $m_{X}$ is a BPA for $X$ such that $m_{X}(x) > 0$. If we have a conditional BPA 
$m_{Y|x}$ for $Y$ given $X = x$ as follows: 
\begin{align}
m_{Y|x}({y}) &= 0.8, \mbox{ and }\nonumber\\
m_{Y|x}(\Omega_{Y}) &= 0.2, 
\end{align}
then its conditional embedding into BPA $m_{x, Y}$ for $\{X, Y\}$ is  
\begin{align}
m_{x, Y}(\{(x, y), (\bar{x}, y), (\bar{x}, \bar{y})\}) &= 0.8, \mbox{ and }\nonumber\\ 
m_{x, Y}(\Omega_{\{X, Y\}}) &= 0.2. 
\end{align}
\end{example}

%
%

There are some differences with conditional probability distributions. 
First, in probability theory, $P_{Y|X}$ consists of \emph{all} conditional distributions $P_{Y|x}$ that are well-defined, i.e., 
for all $x \in \Omega_{X}$ such that $P_{X}(x) > 0$. In D-S belief function theory, we do not have similar constraints. 
We can include only those non-vacuous conditionals $m_{Y|x}$ such that $m_{X}(\{x\}) > 0$. 
Also, if we have more than one conditional BPA for $Y$, given, say for $X = x_{1}$, and $X = x_{2}$ 
(assuming $m_{X}(\{x_1\}) > 0$, and $m_{X}(\{x_2\}) > 0$), we embed these two conditionals for $Y$ to get  BPAs 
$m_{x_1, Y}$ and $m_{x_2, Y}$ for $\{X, Y\}$, and then combine them using Dempster's rule of combination to obtain 
one conditional BPA $m_{Y|X} = m_{x_1, Y} \oplus m_{x_2, Y}$, which corresponds to $P_{Y|X}$ in probability theory.

Second, given any joint PMF $P_{X, Y}$ for $\{X, Y\}$, we can always factor this into 
$P_{X, Y}^{\downarrow X} = P_{X}$ for $X$, and $P_{Y|X}$ for $\{X, Y\}$, such that $P_{X, Y} = P_{X} \otimes P_{Y|X}$. 
This is not true in D-S belief function theory. Given a joint BPA $m_{X, Y}$ for $\{X, Y\}$, we cannot always find a BPA 
$m_{Y|X}$ for $\{X, Y\}$ such that $m_{X, Y} = m_{X, Y}^{\downarrow X} \oplus m_{Y|X}$. However, we can always \emph{construct} joint 
BPA $m_{X, Y}$ for $\{X, Y\}$ by first assessing $m_{X}$ for $X$, and assessing conditionals $m_{Y|x_i}$ for $Y$ 
for those $x_i$ that we have knowledge about and such that $m_{X}(\{x_i\}) > 0$, 
embed these conditionals into  BPAs for $\{X, Y\}$, and combine all such BPAs to obtain the  
BPA $m_{Y|X}$ for $\{X, Y\}$. An implicit assumption here is that BBAs $m_{x_i, Y}$ are distinct, and it is acceptable to combine them using 
Dempster's rule. We can then construct $m_{X, Y} = m_{X} \oplus m_{Y|X}$.

\subsection{Semantics of D-S belief function}
\label{subsec:semantics}

In D-S theory, belief functions are representations of an agent's state of knowledge based on some evidence. As explained by Shafer  \cite{Shafer1981}, such representations can be constructed by comparing the available evidence with  a hypothetical situation in which we receive a coded message, the meaning of which is random. More precisely, assume that a source sends us an encrypted message using a code selected at random from a set of codes $C=\{c_1,\ldots,c_n\}$ with known probabilities $p_1,\ldots,p_n$.  If we decode the message with code $c_i$, we get a decoded message of the form ``$X\in \Gamma(c_i)$'', where $\Gamma$ is a multi-valued mapping from $C$ to $2^{\Omega_X}$. For any nonempty subset $\sfa$ of $\Omega_X$, the probability that the meaning of the original message is ``$X\in \sfa$'' is
\[
m(\sfa)=\sum_{i=1}^n p_i I(\Gamma(c_i)=\sfa),
\]
where $I(\cdot)$ is the indicator function. The random message metaphor thus provides a way to construct BPAs $m$. The fundamental assumption of D-S theory is that such metaphors provide a scale of canonical examples to which any piece of evidence  (or, at least, most pieces of evidence encountered in practice) can be meaningfully compared. 

The random set metaphor accounts for the use of Dempster's rule, which can be easily derived from the assumption that the two BPAs $m_1$ and $m_2$ are induced by stochastically independent randomly coded messages. Two bodies of evidence are considered as independent if ``(they) are sufficiently unrelated that pooling them is like pooling stochastically independent randomly coded messages'' \cite[Section 5.1]{Shafer1981}.

In D-S theory, any belief function can thus be thought of as being induced by a multi-valued mapping from a probability space to the power set of the frame of discernment. Such multi-valued mappings were already explicitly constructed from a statistical model in Dempster's original application of belief functions to statistical inference \cite{Dempster1967}. A statistical model is no more ``real'' than a random code canonical examples: both are idealizations that allow us to formalize our knowledge and make inferences based on  reasonable assumptions. 

\section{A Utility Theory for D-S Belief Function Theory}
\label{sec:dsutility}
In this section, we describe a new utility theory for lotteries where the uncertainty is described by D-S belief functions. These lotteries, called \emph{belief function lotteries}\footnote{This notion was previously introduced in \cite{Denoeux2019} under the name ``evidential lottery.''}, will be introduced in Section \ref{subsec:bflotteries}. We present and discuss assumptions in Section \ref{subsec:assumptions} and state a representation theorem in Section \ref{subsec:theorems}. In Section \ref{subsec:add_assump}, we show that an additional assumption leads to a simpler model and we state the corresponding representation theorem. Finally, in Section \ref{subsec:simplermodel}, we describe an even simpler practical model allowing us to assess the utility of a belief function lottery based on a limited number of parameters.

\subsection{Belief function lotteries}
\label{subsec:bflotteries}
We generalize the decision framework outlined in Section \ref{sec:vNM} by assuming that uncertainty about the state of nature $X$ is described by a BPA $m_X$ for $X$. The probabilistic framework is recovered as a special case when $m_X$ is Bayesian. The BPA $m_X$ is assumed to be given, and is assumed to be a meaningful representation of the DM's state of knowledge about $X$ at a given time, with the semantics described in Section \ref{subsec:semantics}. As before, we define an act as a mapping $f$ from $\Omega_X$ to the set $\bO$ of outcomes. Mapping $f$ pushes  $m_X$ forward  from $\Omega_X$ to $\bO$, transferring each mass $m_X(\sfa)$ for $\sfa \in 2^{\Omega_X}$ to the image of subset $\sfa$ by $f$, denoted as $f[\sfa] = \{f(x): x \in \sfa\}$. The resulting BPA $m$ for $\bO$ is then defined as
\begin{equation}
\label{eq:bf_lottery}
m(\sfb)= \sum_{\{\sfa \in 2^{\Omega_X} \mid f[\sfa]= \sfb\}} m_X(\sfa),
\end{equation}
for all $\sfb\subseteq \bO$ \cite{Denoeuxetal2019}. Eq. \eqref{eq:bf_lottery} clearly generalizes Eq. \eqref{eq:prob_lottery}. The pair $[\bO, m]$ will be called a belief function (bf) lottery. It is a representation of the DM's subjective beliefs about the outcome that will  obtain as a result of selecting act $f$. As noted in \cite{Denoeux2019}, a bf lottery can also arise from a BPA $m_X$ on  $\Omega_X$ and a \emph{nondeterministic} act $f$, defined as mapping from $\Omega_X$ to $2^\bO$. This formalism may be useful to account for under-specified decision problems in which,  for instance, the set of acts or the state space $\Omega_X$ are too coarsely defined to allow for a precise description of the consequences of an act \cite{ghirardato01}.

As before, we assume that two acts can be compared from what we believe their outcomes will be, irrespective of the evidence on which we base our  beliefs. This assumption is a form of what Wakker \cite{Wakker2000} calls the \emph{principle of complete ignorance} (PCI). It implies that two acts resulting in the same bf lottery are equivalent. The problem of expressing preferences between acts becomes that of expressing preferences between bf lotteries.

\begin{remark}
As a consequence of the PCI,  preferences between acts do not depend on the cardinality of the state space $\Omega_X$ in case of complete ignorance. For instance, assume that we  define $\Omega_X=\{x_1,x_2\}$, and we are completely ignorant of the state of nature, so that our belief state is described by the vacuous BPA $m_X(\Omega_X)=1$. Consider two acts $f_1$ and $f_2$ that yield $\$100$ if, respectively, $x_1$ or $x_2$ occurs, and $\$0$ otherwise. These two acts induce the same vacuous bf lottery $m(\bO)=1$ with $\bO=\{\$100,\$0\}$: consequently, they are equivalent according to the PCI. Now, assume that we decide to express the states of nature with finer granularity and we refine state $x_1$ into two states $x_{11}$ and $x_{12}$. Let $\Omega_{X'}=\{x_{11},x_{12},x_2\}$ denote the refined frame. We still have $m_{X'}(\Omega_{X'})=1$ and $m(\bO)=1$, so that our preferences between acts $f_1$ and $f_2$ are unchanged. We note that a Bayesian DM applying Laplace's  principle of indifference (PI) would reach a different conclusion: before the refinement, the PI implies $p_X(x_1)=p_X(x_2)=1/2$, which results in the same probabilistic lottery $\bp=(1/2,1/2)$ on $\bO=\{\$100,\$0\}$ for the two acts, but after the refinement the same principle gives us $p_X(x_{11})=p_X(x_{12})=p_X(x_2)=1/3$; this results in two different lotteries  $\bp_1=(2/3,1/3)$ for act $f_1$ and $\bp_2=(1/3,2/3)$ for act $f_2$, which makes $f_1$ strictly preferable to $f_2$. Considering that the granularity of the state space  is often partly arbitrary $($as discussed by Shafer in \cite{Shafer1976}$)$, we regard this property of invariance to refinement under complete ignorance as a valuable feature of a decision theory based on D-S belief functions.
\end{remark}

We are thus concerned with a DM who has preferences on $\mathcal{L}_{bf}$, the set of all bf lotteries. 
Our task is to find a utility function $u: \mathcal{L}_{bf} \to [\mathbb{R}]$, where  $ [\mathbb{R}]$ denotes the set of closed real intervals, such that the $u(L) = [u, 1-v]$ is viewed as an interval-valued utility of $L$. The interval-valued utility can be interpreted as follows: $u$ and $v$ are, respectively, the degrees of belief of receiving the best and the worst outcome in a bf reference lottery equivalent to $L$ (and $1-v$ is, consequently, the degree of plausibility of receiving the best outcome). Given two lotteries $L$ and $L'$, $L$ is preferred to $L'$  if and only if $u \ge u'$ and $v\le v'$. This leads to incomplete preferences on the set of all bf lotteries. If we assume $u=1-v$ for all bf lotteries, then we have a real-valued utility function on $\mathcal{L}_{bf}$, and consequently, complete preferences.

\begin{example}[Ellsberg's Urn]
\label{ex:ellsbergurn}
Ellsberg \cite{Ellsberg1961} describes a decision problem that questions the adequacy of the vN-M axiomatic framework. Suppose we have an urn with 90 balls, of which 30 are red, and the remaining 60 are either black or yellow. We draw a ball at random from the urn. Let $X$ denote the color of the ball drawn, with $\Omega_X = \{r, b, y\}$. Notice that the uncertainty of $X$ can be described by a BPA $m_X$ for $X$ such that $m_X(\{r\}) = 1/3$, and $m_X(\{b, y\}) = 2/3$.

First, we are offered a choice between Lottery $L_1$: $\$100$ on red, and  Lottery $L_2$: $\$100$ on black, i.e., in $L_1$, you get $\$100$ if the ball drawn is red, and $\$0$ if the ball drawn is black or yellow, and in $L_2$, you get $\$100$ if the ball drawn is black and $\$0$ if the ball drawn is red or yellow. Choice of $L_1$ can be denoted by alternative $f_1: \Omega_X \to \{\$100, \$0\}$ such that $f_1(r) = \$100$, $f_1(b) = f_1(y) = \$0$. Similarly, choice of $L_2$ can be denoted by alternative $f_2: \Omega_X \to \{\$100, \$0\}$ such that $f_2(b) = \$100$, $f_2(r) = f_2(y) = \$0$. $L_1$ can be represented by the BPA $m_1$ for $\bf{O}$ = $\{\$0, \$100\}$ as follows: $m_1(\{\$100\})=1/3$, $m_1(\{\$0\})=2/3$. $L_2$ can be represented by BPA $m_2$ for $\bf{O}$ as follows: $m_2(\{\$0\}) = 1/3$, $m_2(\{\$0, \$100\}) = 2/3$. Notice that $L_1$ and $L_2$ are bf lotteries. Ellsberg notes that a frequent pattern of response is $L_1$  preferred to $L_2$.

Second, we are offered a choice between $L_3$: $\$100$ on red or yellow, and $L_4$: $\$100$ on black or yellow, i.e., in $L_3$ you get $\$100$ if the ball drawn is red or yellow, and $\$0$ if the ball drawn is black, and in $L_4$, you get $\$100$ if the ball drawn is black or yellow, and $\$0$ if the ball drawn is red. $L_3$ can be represented by BPA $m_3$ as follows: $m_3(\{\$100\}) = 1/3$, and $m_3(\{\$0,\$100\}) = 2/3$, and $L_4$ can be represented by the BPA $m_4$ as follows: $m_4(\{\$0\}) = 1/3$, $m_4(\{\$100\}) = 2/3$. $L_3$ and $L_4$ are also belief function lotteries. Ellsberg notes that  $L_4$ is often strictly preferred to $L_3$. Also, the same subjects who prefer $L_1$ to $L_2$, prefer $L_4$ to $L_3$. Table $\ref{tab:4bflotteries}$ is a summary of the four bf lotteries.

Thus, if the outcomes of a lottery are based on the states of a random variable $X$, which is described by a BPA $m_X$ for $X$, then we have a belief function lottery. In this example, we have only two outcomes, 
$\$100$, and $\$0$. $L_1$ and $L_4$ can also be regarded as probabilistic lotteries as the corresponding BPAs are Bayesian. $L_2$ and $L_3$ have BPAs with non-singleton focal sets. Thus, these two lotteries can be considered as involving ``ambiguity'' as the exact distribution of the probability (of $2/3$) between outcomes $\$100$ and $\$0$ is unknown. Regardless of how the probability of $2/3$ is distributed between $b$ and $y$, the preferences of subjects violate the tenets of vN-M utility theory. 
\end{example}

\begin{table}
\caption{Four belief function lotteries in Example \ref{ex:ellsbergurn}}
\begin{center}
\begin{tabular}{|l c c c |}
\hline
$Lottery$ &$m_i(\{\$100\})$ &$m_i(\{\$0\})$ &$m_i(\{\$100, \$0\})$\\
\hline
$L_1$ (\$$100$ on $r$)			&$1/3$ 	&$2/3$		&\\
$L_2$ (\$$100$ on $b$)		&		&$1/3$		&$2/3$\\
\hline
$L_3$ (\$$100$ on $r$ or $y$)	&$1/3$	&			&$2/3$\\
$L_4$ (\$$100$ on $b$ or $y$)	&$2/3$	&$1/3$		&\\
\hline
\end{tabular}
\end{center}
\label{tab:4bflotteries}
\end{table}

\subsection{Assumptions of our framework}
\label{subsec:assumptions}
As in the probabilistic case, we will assume that a DM's preferences for bf lotteries are reflexive and transitive. 
However, unlike the probabilistic case (Assumption \ref{a4p}), we do not assume that these preferences  are complete. 
In the probabilistic case, incomplete preferences are studied in \cite{Aumann1962}, and in the case of sets of utility functions, 
in \cite{Dubraetal2004}.

Our first assumption is identical to Assumption \ref{a1p}.

\begin{assumption}[Weak ordering of outcomes] 
\label{a1b}
The DM's preferences $\succsim$ for outcomes in $\bO = \{O_1, \ldots, O_r\}$ are complete and transitive.
\end{assumption}

This allows us to label the outcomes such that 
\begin{equation}
\label{eq:order}
O_1 \succsim O_2 \succsim \cdots \succsim O_r, \quad \text{and} \quad O_1 \succ O_r.
\end{equation}

Let $\mathcal{L}_{bf}$ denote the set of all bf lotteries on $\bf{O}$ = $\{O_1, \ldots O_r\}$, where the outcomes satisfy Eq. \eqref{eq:order}. 
As every BPA $m$ on $\bf{O}$ is a bf lottery, $\mathcal{L}_{bf}$ is 
essentially the set of all BPAs on $\bf{O}$. As the set of all BPAs include Bayesian BPAs, the set $\mathcal{L}_{bf}$ is a superset of $\mathcal{L}$, 
i.e., every probabilistic lottery on $\bf{O}$ can be considered a bf lottery.

Consider a compound lottery $[\textbf{L}, m]$, where $\textbf{L} = \{L_1, \ldots, L_s\}$, $m$ is a BPA for $\textbf{L}$, and
$L_j = [\bO, m_{j | L_j}]$  is a bf lottery on $\bO$, where $m_{j | L_j}$ is a conditional BPA for $\bO$ in the second stage given that lottery $L_j$ is realized in the first stage. We thus have $s+1$ pieces of evidence represented by BPAs $m$ and $m_{1 | L_1},\ldots,m_{s | L_s}$.  Assuming these pieces of evidence to be independent, they can be combined by Dempster's rule \eqref{Dempster1}, after conditionally embedding the conditional BPAs $m_{j | L_j}$ (see Section \ref{subsec:conditional_bf}). Marginalizing the orthogonal sum on $\bO$, we get a BPA $m'$.  Assumption \ref{a2b} below posits that the resulting simple lottery $[\bO, m']$ is equally preferred to the original compound lottery $[\textbf{L}, m]$, i.e.,  we can reduce the compound lottery to a simple bf lottery on $\bO$ using the D-S calculus.

\begin{assumption}[Reduction of compound lotteries]
\label{a2b}
Suppose $[\textbf{L}, m]$ is a compound lottery as described in the previous paragraph. Then, 
$[\textbf{L}, m] \sim [\bO, m^\prime]$, where 
\begin{equation}
\label{eq:reduction}
m^\prime = \left(m \oplus \left(\bigoplus_{j=1}^{s} m_{L_j, j}\right)\right)^{\downarrow \bO},
\end{equation}
and $m_{L_j, j}$ is a BPA for $(\textbf{L}, \bO)$ obtained from $m_{j | L_j}$ by conditional embedding, for $j = 1, \ldots, s$. 
\end{assumption}

\begin{figure}[htbp]
\begin{center}
\includegraphics[width=4in]{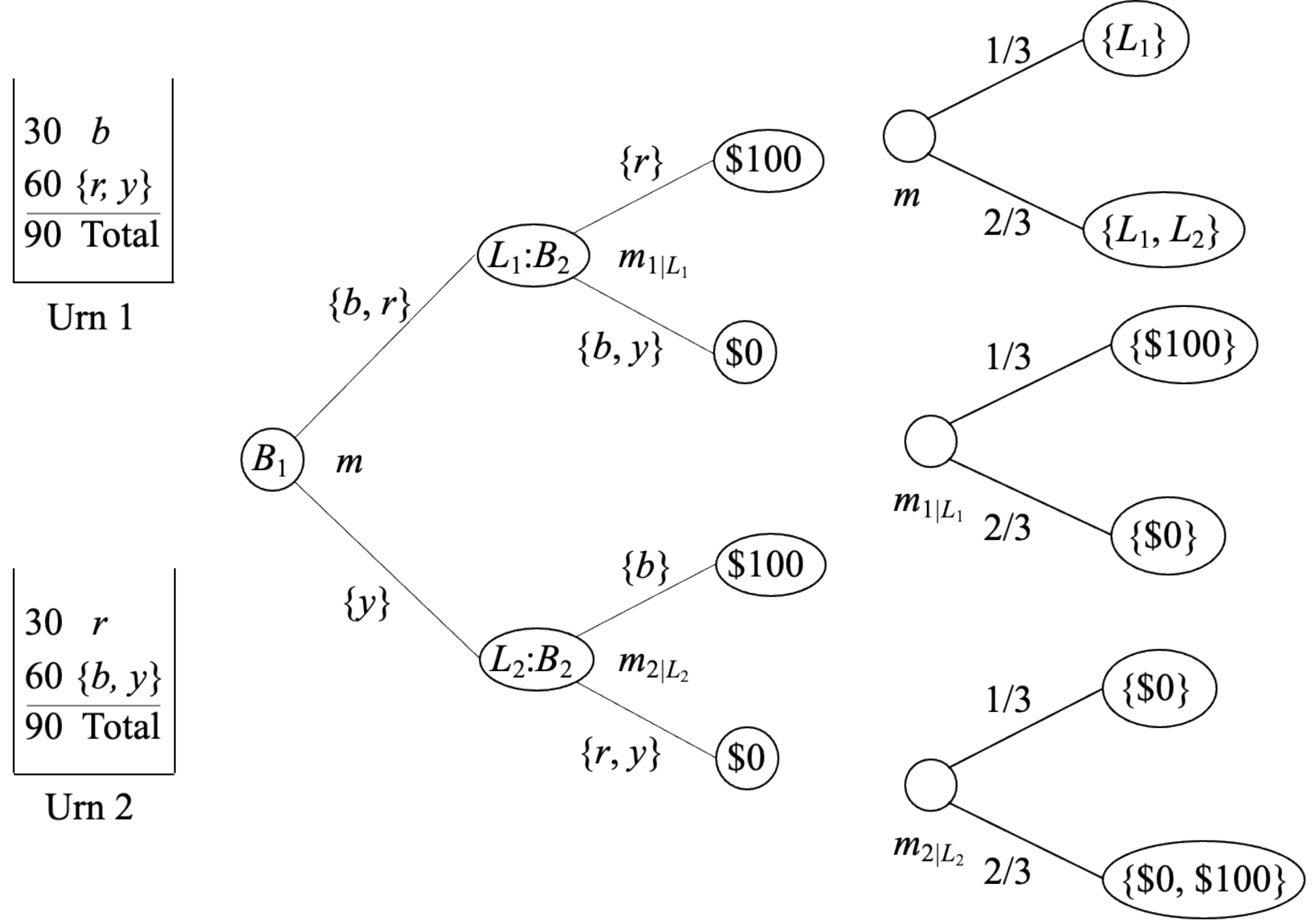}
\caption{A 2-stage compound belief function lottery. $B_1$ is drawn from Urn 1, and $B_2$ is drawn from Urn 2. The corresponding BPAs are shown on the right. In the first stage, $m$ is a BPA for $\{L_1, L_2\}$. In the second stage, $m_{1|L_1}$ and $m_{2|L_2}$ are BPAs for $\{\$100, \$0\}$.}
\label{fig:complotex3}
\end{center}
\end{figure}

\begin{example}
Consider two urns: Urn 1 contains 90 balls, 30 of which are black, and 60 are red or yellow. Urn 2 is identical to Ellsberg's urn in Example  \ref{ex:ellsbergurn}: it contains 90 balls, 30 of which are red, and 60 are black or yellow. Assume that, in the first stage, you are allowed to draw one ball $B_1$ from Urn 1:
\begin{itemize}
\item If $B_1$ is black or red, then you will be allowed to draw one ball  from Urn 2 in the second stage, and you will get \$100 if it is red, and \$0 otherwise (lottery $L_1$ of Example \ref{ex:ellsbergurn});
\item If $B_1$ is yellow, you will be allowed to draw one ball from Urn 2, and you will get \$100 if it is black, and \$0 otherwise (lottery $L_2$ of Example \ref{ex:ellsbergurn}).
\end{itemize}
Here, we have a compound bf lottery with outcome space $\bL=\{L_1,L_2\}$. We get $L_1$ or $L_2$ depending on the color $X_1$ of ball $B_1$ drawn from Urn 1. We have $m_{X_1}(\{b\})=1/3$ and   $m_{X_1}(\{r,y\})=2/3$, and the act $f$ defined by $f(b)=f(r)=L_1$, $f(y)=L_2$. The BPA on $\bL$ is, thus, $m(\{L_1\})=1/3$ and  $m(\{L_1,L_2\})=2/3$. Now, $L_1$ and $L_2$ are bf lotteries on $\bO=\{\$100,\$0\}$, with BPAs 
\begin{equation*}
m_{1|L_1}(\{\$100\})=1/3, \quad m_{1|L_1}(\{\$0\})=2/3
\end{equation*}
and
\begin{equation*}
m_{2|L_2}(\{\$0\})=1/3, \quad m_{2|L_2}(\{\$0,\$100\})=2/3.
\end{equation*}
Here, $m_{j|L_j}$ is a conditional BPA on $\bO$, given that $L_j$ is obtained in the first stage (see Figure \ref{fig:complotex3}).  The conditional embeddings of $m_{1|L_1}$ and $m_{2|L_2}$ are BPAs on $\bL\times\bO$ equal, respectively, to
\begin{align*}
m_{L_1,1}(\{(L_1,\$100),(L_2,\$100), (L_2,\$0)\})&=1/3, \\
m_{L_1,1}(\{(L_1,\$0),(L_2,\$100), (L_2,\$0)\})&=2/3,
\end{align*}
and
\begin{align*}
m_{L_2,2}(\{(L_2,\$0),(L_1,\$100), (L_1,\$0)\})&=1/3, \\
m_{L_2,2}(\bL\times\bO)&=2/3,
\end{align*}
Their orthogonal sum is
\begin{align*}
(m_{L_1,1}\oplus m_{L_2,2})(\{(L_1,\$100), (L_2,\$0)\})&=1/9, \\
(m_{L_1,1}\oplus m_{L_2,2})(\{(L_1,\$0), (L_2,\$0)\})&=2/9, \\
(m_{L_1,1}\oplus m_{L_2,2})(\{(L_1,\$100),(L_2,\$100), (L_2,\$0)\})&=2/9,\\
(m_{L_1,1}\oplus m_{L_2,2})(\{(L_1,\$0),(L_2,\$100), (L_2,\$0)\})&=4/9.
\end{align*}
Combining it with $m$, we get
\begin{align*}
(m\oplus m_{L_1,1}\oplus m_{L_2,2})(\{(L_1,\$100)\})&=3/27, \\
(m\oplus m_{L_1,1}\oplus m_{L_2,2})(\{(L_1,\$0)\})&=6/27, \\
(m\oplus m_{L_1,1}\oplus m_{L_2,2})(\{(L_1,\$100), (L_2,\$0)\})&=2/27, \\
(m\oplus m_{L_1,1}\oplus m_{L_2,2})(\{(L_1,\$0), (L_2,\$0)\})&=4/27, \\
(m\oplus m_{L_1,1}\oplus m_{L_2,2})(\{(L_1,\$100),(L_2,\$100), (L_2,\$0)\})&=4/27,\\
(m\oplus m_{L_1,1}\oplus m_{L_2,2})(\{(L_1,\$0),(L_2,\$100), (L_2,\$0)\})&=8/27.
\end{align*}
Marginalizing on $\bO$, we get $m'=(m\oplus m_{L_1,1}\oplus m_{L_2,2})^{\downarrow \bO}$ equal to
\begin{align*}
m'(\{\$100\})&=3/27=1/9,\\
m'(\{\$0\})&=10/27,\\
m'(\{\$100,\$0\})&=14/27.
\end{align*}
According to Assumption \ref{a2b}, a rational DM should be indifferent between receiving the compound bf lottery $[\bL,m]$, or receiving the bf lottery $[\bO,m']$, i.e., a prize about which the only information he has is given by a randomly coded message whose meaning can be ``The value of the prize is \$100'' with probability 1/9, ``The value of the prize is \$0'' with probability 10/27, and ``The value of the prize is unknown'' with probability 14/27.
\end{example}

The following proposition states that Assumption \ref{a2b} generalizes Assumption \ref{a2p}.

\begin{proposition}
Let $\bL=\{L_1,\ldots,L_s\}$ be a set of bf lotteries, with $L_j=[\bO,m_{j|L_j}]$, in which $m_{j|L_j}$ is a Bayesian conditional BPA for $\bO$ such that $m_{j|L_j}(\{O_i\})=p_i^{(j)}$ and $\sum_{i=1}^r p_i^{(j)}=1$ for $j=1,\ldots,s$. Let $[\textbf{L}, m]$ be a compound lottery in which  $m$ is a Bayesian BPA for $\bL$ such that $m(\{L_j\}=q_j$ for  $j=1,\ldots,s$ with $\sum_{j=1}^s q_j=1$. Then BPA $m'$ defined by \eqref{eq:reduction} is Bayesian and it verifies
\begin{equation}
\label{eq:reduction1}
m'(\{O_i\})= \sum_{j=1}^s q_j p_i^{(j)}
\end{equation}
for $i=1,\ldots,r$.
\end{proposition}
\begin{proof}
The conditional embedding of $m_{j|L_j}$ is given by
\[
m_{L_j,j}\left(\{(L_j,O_i)\} \cup \overline{\{L_j\}}\times \bO\right)=p_i^{(j)}, \quad i=1,\ldots,r.
\]
Let $m_0=\bigoplus_{j=1}^{s} m_{L_j, j}$. It is a  BPA for $\bL\times\bO$ defined by
\[
m_0\left(\{(L_1,O_{i_1}),\ldots,(L_s,O_{i_s}) \} \right)=p^{(1)}_{i_1} \ldots p^{(s)}_{i_s}
\]
for all $(i_1,\ldots,i_s) \in \{1,\ldots,r\}^s$. Combining $m_0$ with $m$, we get a Bayesian BPA $m_0'$ on $\bL\times\bO$  such that
\[
m_0'(\{(L_j,O_i)\})=q_jp_i^{(j)}.
\]
After marginalizing on $\bO$, we finally get Eq. \eqref{eq:reduction1}.
\end{proof}

Next, we define a \emph{bf reference} lottery $[\bO_2, m]$ as a bf lottery on $\bO_2 = \{O_1, O_r\}$. A bf reference lottery has three 
parameters $u=m(\{O_1\})$, $v=m(\{O_r\})$, and $w=m(\bO_2)$, which  are all non-negative and sum to 1. It can be equivalently denoted as $[\bO_2, (u,v,w)]$. The first and second elements of the triple are, respectively,   the degrees of belief of receiving the best and worst outcomes, while the third element can be seen as a degree of ignorance. Obviously, the degrees of plausibility of receiving the best and the worst outcomes are, respectively, $1-v$ and $1-u$. The following assumption states that any deterministic bf lottery is equally preferred to some bf reference lottery.

\begin{assumption}[Continuity] 
\label{a3b}
Any subset of outcomes $\sfa\subseteq \bO$ (considered as a deterministic bf lottery) is indifferent to a bf reference lottery $\tsfa=[\bO_2, (u_{\sfa}, v_{\sfa}, w_{\sfa})]$ 
for some $u_{\sfa}, v_{\sfa}, w_{\sfa} \ge 0$ such that $u_{\sfa} + v_{\sfa} + w_{\sfa} = 1$. Furthermore, $w_\sfa=0$ if $\sfa=\{O_i\}$ is a singleton subset. 
\end{assumption}

For singleton subsets, the equivalent bf reference lottery is Bayesian: this ensures that Assumption \ref{a3b} is a generalization of Assumption \ref{a3p}. For non-singleton subsets $\sfa$ of outcomes, we may have $w_\sfa>0$, i.e., the bf reference lottery may not be Bayesian. In other words, we do not assume that ambiguity can be resolved by selecting an equivalent probabilistic reference lottery. 

\begin{example}
\label{ex:ex3}
Consider lottery $L_2 = [\{\$100, \$0\}, m_2]$ in Example \ref{ex:ellsbergurn}, where $m_2(\{\$0\}) = 1/3$, and $m_2(\{\$100, \$0\}) = 2/3$. 
Suppose we wish to assess the utility of focal set $\{\$100, \$0\}$ using a probabilistic reference lottery $[\{\$100, \$0\}, (p, 1-p)]$. A DM may have 
the following preferences. For any $p \le 0.2$ she prefers $\{\$100, \$0\}$ to the probabilistic reference lottery, and for any $p \ge 0.3$, 
she prefers the probabilistic reference lottery to $\{\$100, \$0\}$. However, she is unable to give us a 
precise $p$ such that $\{\$100, \$0\} \sim [\{\$100, \$0], (p, 1-p)]$. For such a DM, we can assess a bf reference lottery $[\{\$100, \$0\}, (0.2, 0.7, 0.1)]$ 
such that $Bel_{m_\sfa}(\{\$100\})=0.2$ and $Pl_{m_\sfa}(\{\$100\})=0.3$.
\end{example}

\begin{assumption}[Quasi-order]
\label{a4b}
The preference relation $\succsim$ for bf lotteries on $\mathcal{L}_{bf}$ is a quasi-order, i.e., it is reflexive and transitive. 
\end{assumption}

In contrast with the probabilistic case (Assumption \ref{a4p}), we do not assume that $\succsim$ is complete. There are many reasons we may not wish to assume completeness. It is not descriptive of human behavior. Even from a normative point of view, it is questionable that a DM has complete preferences on all possible lotteries. The assumption of incomplete 
preferences is consistent with the D-S theory of belief functions where we have non-singleton focal sets. Several authors, 
such as Aumann \cite{Aumann1962}, and Dubra \emph{et al.} \cite{Dubraetal2004} argue why the assumption of complete preferences may not be 
realistic in many circumstances. 

The substitutability assumption is similar to the probabilistic case (Assumption \ref{a5p})-- we replace an outcome in the probabilistic case by a focal set of $m$ in the 
bf case.
\begin{assumption}[Substitutability]
\label{a5b}
In any bf lottery $L = [\bO, m]$, if we substitute a focal set $\sfa$ of $m$ by an equally preferred bf reference lottery 
$\tsfa=[\bO_2, m_{\sfa}]$, then the result is a compound lottery that is equally preferred to $L$. 
\end{assumption}

It follows from Assumptions \ref{a1b}--\ref{a5b} that given any bf lottery, we can reduce it to an equally preferred bf reference lottery. This is stated as 
Theorem \ref{th:redbflotterygen} below.

\begin{theorem} [Reducing a bf lottery to an indifferent bf reference lottery]
\label{th:redbflotterygen}
Under Assumptions \ref{a1b}--\ref{a5b}, any  bf lottery $L = [\bO, m]$   with focal sets $\sfa_1,\ldots,\sfa_k$ is indifferent to a bf reference  lottery $\tL = [\bO_2, \tm]$, such that
\begin{subequations}
\label{eq:exputilgen}
\begin{align}
\tm(\{O_1\}) &= \sum_{i = 1}^k m(\sfa_i)\,u_{\sfa_i},\\
\tm(\{O_r\}) &= \sum_{i = 1}^k m(\sfa_i)\,v_{\sfa_i},\quad\text{and}\\
\tm(\bO_2) &= \sum_{i = 1}^k m(\sfa_i)\,w_{\sfa_i},
\end{align}
\end{subequations}
where $u_{\sfa_i}$, $v_{\sfa_i}$ and $w_{\sfa_i}$ are the masses assigned, respectively, to $\{O_1\}$, $\{O_r\}$ and $\bO_2$ by the bf reference  lottery $\tsfa_i$ equivalent to $\sfa_i$. 
\end{theorem}

\begin{proof}
From Assumption \ref{a3b} (continuity),   we can replace each focal set $\sfa_i$ of $m$ one at a time by an indifferent bf reference lottery $\tsfa_i = [\bO_2, m_{i|\tsfa_i}]$, yielding  a sequence of compound lotteries.  From Assumptions \ref{a5b} (substitutability) and \ref{a4b} (quasi-order), these compound lotteries are all indifferent to $L$ (see  Figure \ref{fig:theorem2}). 
Let   $L' = [\textbf{L}, m^\prime]$ be the compound lottery obtained after all focal sets $\sfa_i$ have been substituted, with $m^\prime(\{\tsfa_i\}) = m(\sfa_i)$. From Assumption \ref{a2b} (reduction of compound lotteries), $L'$ can be reduced to a reference bf lottery $\tL=[\bO_2,\tm]$, by considering each BBA $m_{i|\tsfa_i}$ as a conditional BPA and applying the rules of D-S calculus. The reduced bf lottery  $\tL=[\bO_2, \tm]$ is then given by
\begin{equation*}
\tm = \left(m^\prime \oplus \left(\bigoplus_{i=1}^k m_{\tsfa, i}\right)\right)^{\downarrow \bO_2},
\end{equation*} 
where $m_{\tsfa, i}$ is the BPA for $\bO_2$ obtained from $m_{i|\tsfa_i}$ by conditional embedding. If we have
\\$m_{i|\tsfa_i}(\{O_1\}) = u_{\sfa_i}$, $m_{i|\tsfa_i}(\{O_r\}) = v_{\sfa_i}$
and $m_{i|\tsfa_i}(\bO_2) = w_{\sfa_i}$, then after conditional embedding BPA $m_{\tsfa_i, i}$ is as follows:
\begin{align*}
m_{\tsfa_i, i}(\{(\tsfa_i, O_1)\} \cup (\overline{\{\tsfa_i\}}\times \bO_2)) &= u_{\sfa_i}\\
m_{\tsfa_i, i}(\{(\tsfa_i, O_r)\} \cup (\overline{\{\tsfa_i\}}\times \bO_2)) &= v_{\sfa_i}\\
m_{\tsfa_i, i}(\textbf{L}\times\bO_{2}) &= w_{\sfa_i}.
\end{align*}
Let $m_0$ denote $\bigoplus_{i=1}^k m_{\tsfa_i, i}$. The focal sets of $m_0$ are of the form
\begin{equation*}
\sfa_{(I_1, I_2, I_3)} = \left(\bigcup_{i \in I_1} \tsfa_i \times \{O_1\}\right) \cup \left(\bigcup_{i \in I_2} \tsfa_i \times \{O_r\}\right) \cup 
\left(\bigcup_{i \in I_3} L_i 
\times \bO_2\right),
\end{equation*}
for all partitions $(I_1, I_2, I_3)$ of $\{1, \dots, k\}$, and the corresponding value of $m_0$ is:
\begin{equation*}
m_0(\sfa_{(I_1, I_2, I_3)}) = \left(\prod_{i\in I_1} u_{\sfa_i}\right)\left(\prod_{i\in I_2} v_{\sfa_i}\right)\left(\prod_{i\in I_3} 
w_{\sfa_i}\right).
\end{equation*}
Next, we combine Bayesian BPA $m^\prime$ for $\textbf{L}$ with $m_0$. The focal sets of $m^\prime \oplus m_0$ are of the 
form $\{\tsfa_i\} \times \{O_1\}$,  
$\{\tsfa_i\} \times \{O_r\}$, and $\{\tsfa_i\} \times \bO_2$ depending on whether $i \in I_1$ or $i \in I_2$, or $i \in I_3$, respectively, with mass 
$m(\sfa_i) m_0(\sfa_{(I_1, I_2, I_3)})$. Finally we marginalize $m^\prime \oplus m_0$ to $\bO_2$. The  mass assigned to each focal set $\{\tsfa_i\} \times \{O_1\}$ is $m_{\sfa_i} u_{\sfa_i}$. Thus, $\tm(\{O_1\}) = 
\sum_{i=i}^k m(\sfa_i) u_{\sfa_i}$. Similarly, 
$\tm(\{O_r\}) = \sum_{i=i}^k m(\sfa_i) v_{\sfa_i}$, and 
$\tm(\bO_2) = \sum_{i=i}^k m(\sfa_i) w_{\sfa_i}$.
\end{proof}

\begin{figure}
\centering
\includegraphics[width=\textwidth]{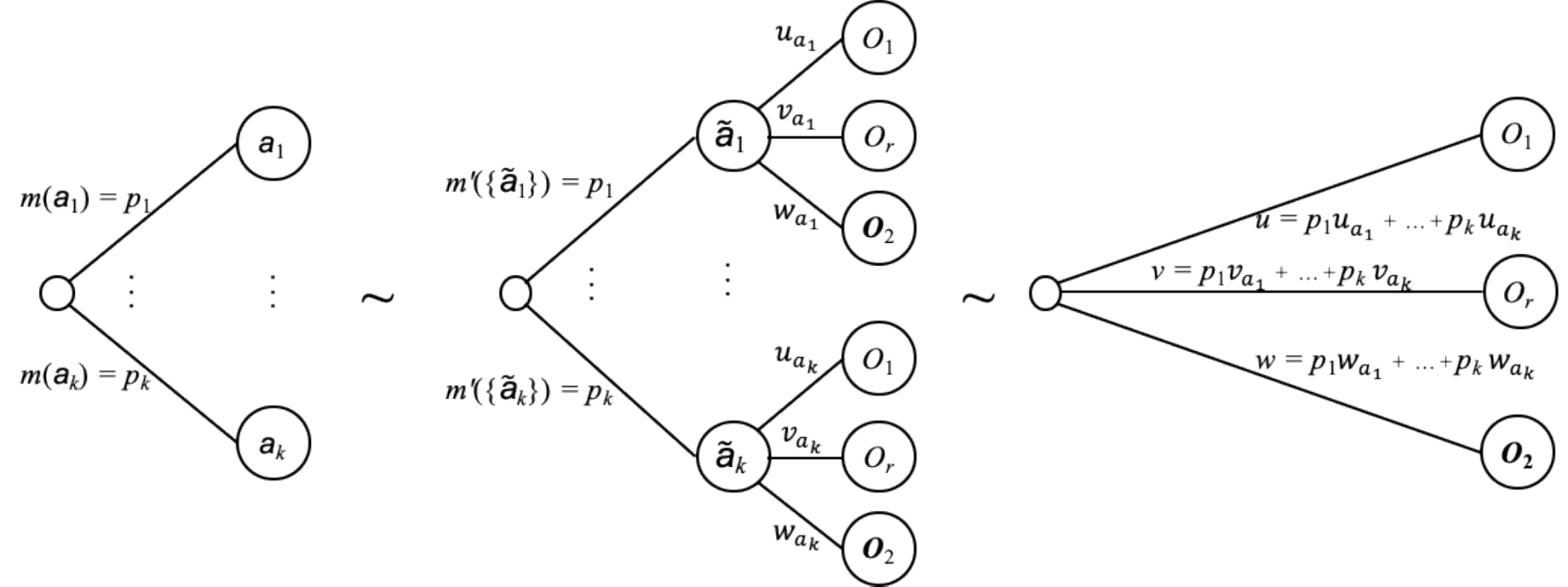}
\caption{Reducing a bf lottery to a bf reference lottery}
\label{fig:theorem2}
\end{figure}


Next, we formulate a monotonicity assumption to generalize Assumption \ref{a6p}.  Given two bf reference lotteries, if any of them assigns a higher degree of belief to the best outcome $O_1$ and a lower degree of belief to the worst outcome $O_r$ (or, equivalently,  higher degrees of belief and plausibility to  $O_1$), then it should arguably be preferred. If this is not the case, i.e., if the best and the worst outcomes both have a higher degree of belief for one lottery (or, equivalently, if the belief-plausibility intervals for $O_1$ in the two lotteries are strictly nested), then there does not seem to be any solid ground for preference, and the two lotteries can be considered as incomparable. This line of reasoning is formalized in Assumption \ref{a6b} below.

\begin{assumption}[Monotonicity]
\label{a6b}
Suppose $L = [\bO_2, (u,v,w)]$ and $L^\prime = [\bO_2, (u',v',w')]$ are bf reference lotteries. Then,
$L \succsim L^\prime$ if and only if $u \ge u' $ and $1-v \ge 1-v'$. 
\end{assumption}

It is clear that $\succsim$ as defined in Assumption \ref{a6b} is reflexive and transitive.  The  corresponding indifference relation is $L\sim L'$ if and only if $u=u'$ and $v=v'$, i.e., if and only if $L=L'$. Also,  the preference relation $\succsim$ on the set of all bf reference lotteries is obviously incomplete: two lotteries $L$ and $L'$ are \emph{incomparable} if not $L \succsim L^\prime$ and not $L^\prime \succsim L$, i.e., if the intervals $[u,1-v]$ and $[u',1-v']$ are strictly nested.


The preference relation defined in Assumption \ref{a6b} can equivalently be expressed as  $L \succsim L^\prime$ if and only if 
$Bel_m(\{O_1\}) \ge Bel_{m'}(\{O_1\})$ and $Pl_m(\{O_1\}) \ge Pl_{m'}(\{O_1\})$ (meaning that  the best outcome $O_1$ is deemed both more credible and more plausible under $L$ than it is if under $L'$). 

We note that a stronger notion of preference would be to prefer $L$ over $L'$ if and only if  $Bel_m(\{O_1\}) \ge Pl_{m'}(\{O_1\})$, i.e., if and only if $u \ge 1-v'$.
This alternative preference relation is arguably too strict, which would lead to a more incomplete preference order on lotteries.

To conclude this section, we note that Assumptions \ref{a1b}, \ref{a3b} and \ref{a6b} imply  the following consistency constraints between the reference bf lotteries equivalent to  single outcomes:
\begin{equation}
\label{eq:consistency1}
1 = u_{\{O_1\}} \ge u_{\{O_2\}} \ge \ldots \ge u_{\{O_r\}} = 0.
\end{equation}


\subsection{Representation theorem}
\label{subsec:theorems}

\begin{theorem}[Interval-valued utility function]
\label{th:dsutility1}
Suppose $L = [\bO, m]$ and $L^\prime = [\bO, m^\prime]$ are bf lotteries on $\bO$. 
If the preference relation $\succsim$ on $\mathcal{L}_{bf}$ satisfies Assumptions \ref{a1b}--\ref{a6b}, then there are intervals $[u_{\sfa}, 1-v_{\sfa}]$ 
associated with nonempty subsets $\sfa\subseteq \bO$ such that  $L \succsim L'$ if and only if
\begin{subequations}
\begin{equation}
\label{eq:lowerutil}
\sum_{\emptyset \neq \sfa\subseteq \bO} m(\sfa)\,u_{\sfa} \ge \sum_{\emptyset \neq \sfa\subseteq \bO} m^\prime(\sfa)\,u_{\sfa}
\end{equation}
and
\begin{equation}
\label{eq:upperutil}
\sum_{\emptyset \neq \sfa\subseteq \bO} m(\sfa)\,v_{\sfa} \le \sum_{\emptyset \neq \sfa\subseteq \bO} m^\prime(\sfa)\,v_{\sfa}.\end{equation}
\end{subequations}
Thus, for a bf lottery $L = [\bO, m]$, we can define 
\begin{equation}
\label{eq:dsutility1}
[u](L) = [u, 1-v]
\end{equation}
as an interval-valued utility of $L$, with
\begin{equation}
\label{eq:uw}
u=\sum_{\emptyset \neq \sfa\subseteq \bO} m(\sfa)\,u_{\sfa_i} \quad \text{and} \quad  v=\sum_{\emptyset \neq \sfa\subseteq \bO} m(\sfa)\,v_{\sfa}.
\end{equation}

 Also, such a utility function is unique up to a strictly increasing affine transformation, 
i.e., if $u' = a\,u + b$, and $v' = a\,v + b$, 
where $a >0$, and $b$ are real constants, then 
\begin{equation*}
[u'](L) = [u', 1-v']
\end{equation*}
also qualifies as an interval-valued utility function. 
\end{theorem}

\begin{proof}
The proof is immediate from Theorem \ref{th:redbflotterygen} and Assumption \ref{a6b} (monotonicity).
\end{proof}

A special case of Theorem \ref{th:dsutility1} is if we use Bayesian bf reference lotteries for the continuity assumption, i.e., $w_{\sfa} = 0$ for all focal sets $\sfa$ of $m$. In this case, Theorem \ref{th:dsutility1} implies Corollary \ref{th:dsutility2} below where we have a real-valued utility function and a complete ordering on $\mathcal{L}_{bf}$.

\begin{corollary}[Real-valued utility function]
\label{th:dsutility2}
Suppose $L = [\bO, m]$ and $L^\prime = [\bO, m^\prime]$ are bf lotteries on $\bO$. If the preference relation $\succsim$ on $\mathcal{L}_{bf}$ satisfies Assumptions \ref{a1b}--\ref{a6b} and if $w_{\sfa} = 0$ for all focal sets $\sfa$ of $m$ and $m^\prime$, then there are numbers $u_{\sfa}$  associated with nonempty subsets $\sfa \subseteq \bO$ such that $L_1 \succsim L_2$ if and only if
\begin{equation*}
\sum_{\emptyset \neq \sfa\subseteq \bO} m(\sfa)\,u_{\sfa} \ge \sum_{\emptyset \neq \sfa\subseteq \bO} m^\prime(\sfa)\,u_{\sfa}.
\end{equation*}
Thus, for a bf lottery $L = [\bO, m]$, we can define 
\begin{equation}
\label{eq:dsutility}
u(L) = \sum_{\emptyset \neq \sfa\subseteq \bO} m(\sfa)\,u_{\sfa}
\end{equation}
as the utility of $L$. Also, such a utility function is unique up to a strictly increasing affine transformation, i.e., if $u_{\sfa}' = a\,u_{\sfa} + b$, where $a >0$, and $b$ are real constants, then 
\begin{equation*}
u'(L) = \sum_{\emptyset \neq \sfa\subseteq \bO} m(\sfa)\,u_{\sfa}'
\end{equation*}
also qualifies as a utility function.
\end{corollary}

\begin{proof}
The result in Corollary \ref{th:dsutility2} follows trivially from Theorem \ref{th:dsutility1}.
\end{proof}

The utility function in Eq.  \eqref{eq:dsutility} has exactly the same form as Jaffray's linear utility \cite{Jaffray1989}. This is discussed further in Section \ref{subsec:Jaffray}. 

Next, we  illustrate the application of Theorem \ref{th:dsutility1} to some examples: Ellsberg's urn problem described in Example \ref{ex:ellsbergurn},  the one red ball problem described in \cite{JirousekShenoy2017}, and the 1,000 balls urns described in \cite{BeckerBrownson1964}. 

\begin{figure}[htbp]
\begin{center}
\includegraphics[width=6in]{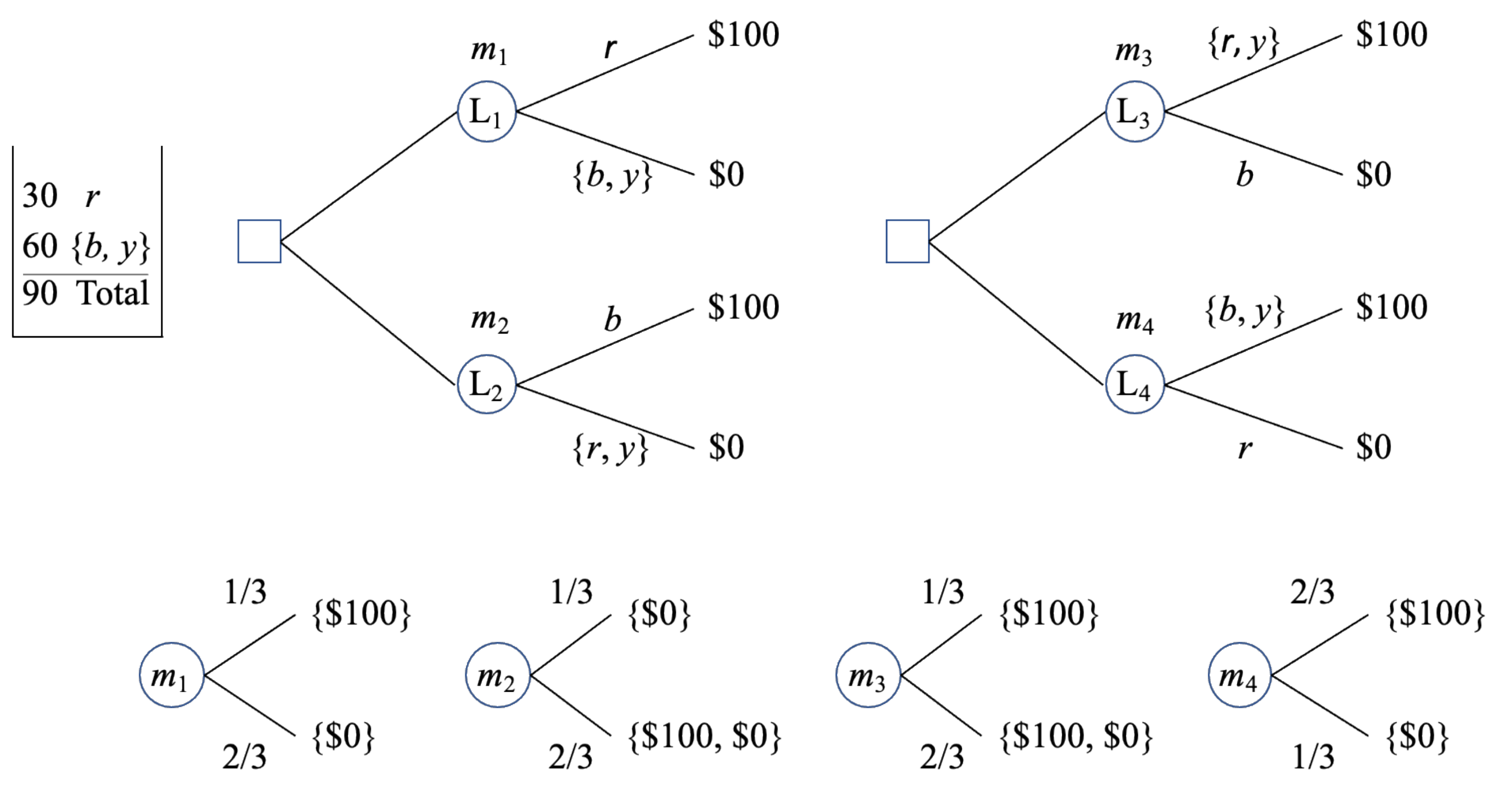}
\caption{Ellsberg's urn, choice of lotteries, and the corresponding belief function lotteries}
\label{fig:ellsbergurn}
\end{center}
\end{figure}

\begin{example}[Ellsberg's urn]
\label{ex:ellsbergurn3}
Consider the four bf lotteries described in Example \ref{ex:ellsbergurn} (also in Table \ref{tab:4bflotteries} and in Figure \ref{fig:ellsbergurn}). Given a vacuous bf lottery $[\{\$100, \$0\}, \iota]$, where $\iota$ is the vacuous BPA on $\bO=\{\$100, \$0\}$, what is an indifferent bf reference lottery? For an ambiguity-averse DM,  
\[
[\{\$100, \$0\}, (1/2, 1/2, 0)] \succ [\{\$100, \$0\}, \iota].
\]
For such a DM, we thus have   $1-v_{\{\$100,\,\$0\}}  < 1/2$.

For the first choice problem between $L_1$ $(\$100$ on $r)$ and $L_2$ $(\$100$ on $b)$, using Eq. \eqref{eq:dsutility}, 
$[u](L_1) = [1/3,1/3]$, and 
\[
[u](L_2) = \frac23\left[u_{\{\$100,\,\$0\}}, 1-v_{\{\$100,\,\$0\}} \right]. 
\]
Thus, an ambiguity-averse DM would choose $L_1$. This result is valid as long as $1-v_{\{\$100,\,\$0\}}  < 1/2$ and is consistent with Ellsberg's findings. 
For the second choice problem between $L_3$ $(\$100$ on $r$ or $y)$, and $L_4$ $(\$100$ on $b$ or $y)$, 
\[
[u](L_3) = \frac13(1) + \frac23\left[u_{\{\$100,\,\$0\}}, 1-v_{\{\$100,\,\$0\}} \right],
\] 
and $[u](L_4) = [2/3,2/3]$. An ambiguity-averse DM would choose $L_4$, as
\[
\frac13 +  \frac23 (1-v_{\{\$100,\,\$0\}}) < \frac23
\] 
as long as $1-v_{\{\$100,\,\$0\}}  < 1/2$, a result that is also consistent with Ellsberg's empirical findings.
\end{example}

\begin{figure}[htbp]
\begin{center}
\includegraphics[width=4in]{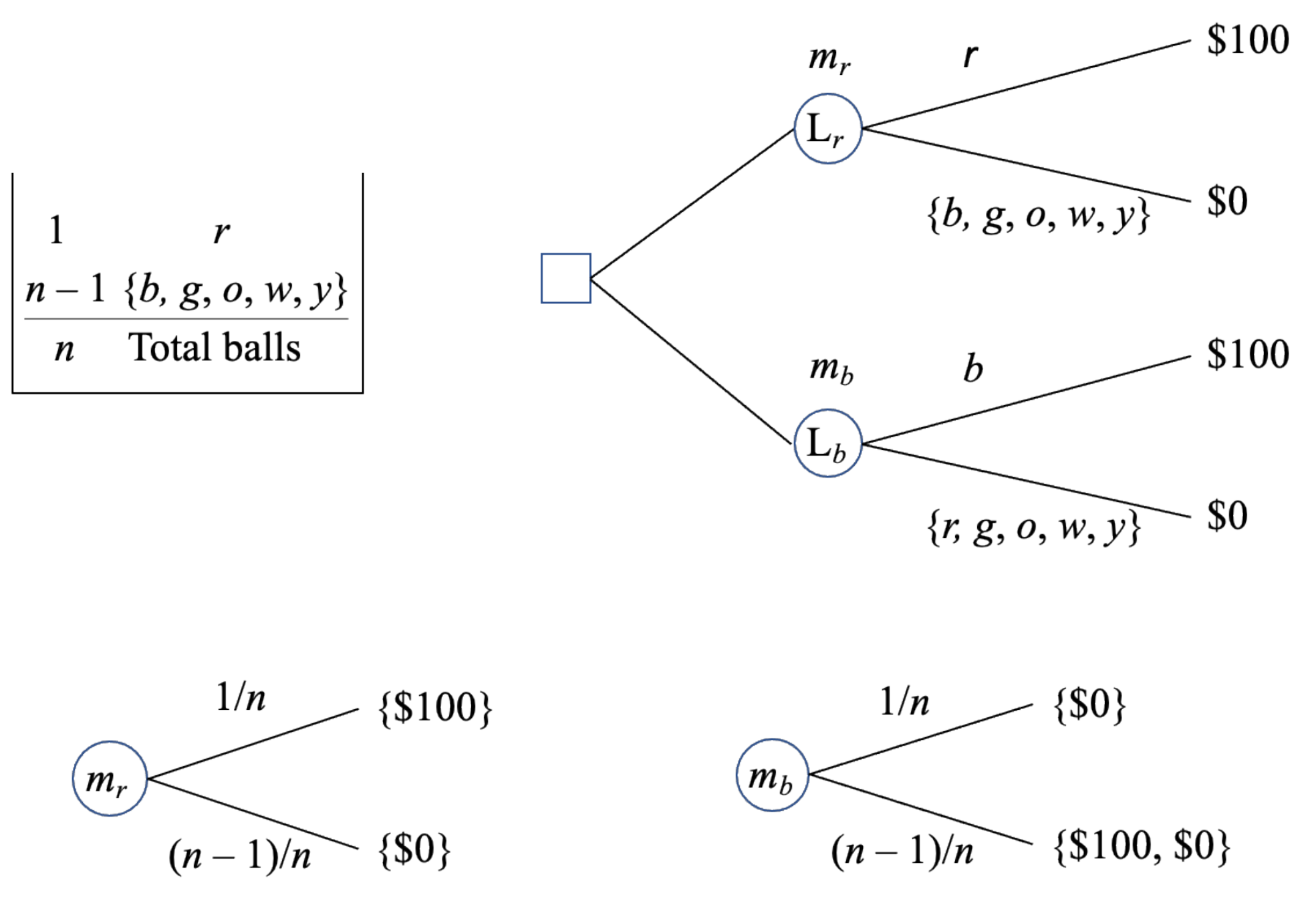}
\caption{One red ball urn, choices, and the corresponding belief function lotteries}
\label{fig:oneredball}
\end{center}
\end{figure}

\begin{example}[One red ball]
\label{ex:oneredball}
Consider the following example called ``one red ball'' in \cite{JirousekShenoy2017} (see Figure \ref{fig:oneredball}). An urn possibly contains balls of six colors: 
red $(r)$, blue $(b)$, green $(g)$, orange $(o)$, white $(w)$, and yellow $(y)$. One ball is drawn at random from the urn. 
We are informed that the urn has a total of $n$ balls, where $n$ is a positive integer, and that there is exactly one red ball in the urn. 
Suppose random variable $X$ denotes the color of the ball drawn from the urn. Then $\Omega_X = \{r, b, g, o, w, y\}$, and $m_X$ is a BPA for $X$ 
such that $m_X(\{r\}) = 1/n$, and $m_X(\{b, g, o, w, y\}) = (n-1)/n$. First, you choose a color, and then you draw a ball at random from the urn. You win $\$100$ 
if the color of the ball drawn from the urn matches the color you chose, and you win $\$0$ if it doesn't. What color do you choose? 
In \cite{JirousekShenoy2017}, the authors  describe some informal experiments where all respondents chose red for $n \le 7$, and for $n \ge 8$, several 
respondents preferred a color different from red.

Suppose you  choose $r$. The bf lottery $L_r$ based on $m_X$ is as follows: $[\{\$100, \$0\}, m_r]$, where $m_r(\{\$100\}) = 1/n$, 
and $m_r(\{\$0\}) = (n-1)/n$. If the color you pick is $b$, then the bf lottery $L_b$ is $[\{\$100, \$0\}, m_b]$, where $m_b(\{\$0\}) = 1/n$, and $m_b(\{\$100, \$0\}) = (n-1)/n$. Thus, we have $[u](L_r)=[1/n,1/n]$, and 
\[
[u](L_b)=\frac{n-1}n \left[ u_{\{\$100,\,\$0\}}, 1-v_{\{\$100,\,\$0\}} \right].
\]
So, $L_b$ is strictly preferred to $L_r$ whenever
\[
\frac{n-1}n u_{\{\$100,\,\$0\}} > \frac1n,
\]
i.e., whenever $u_{\{\$100,\,\$0\}} > 1/(n-1)$, and $L_r$ is strictly preferred to $L_b$ whenever
\[
\frac{n-1}n \left(1-v_{\{\$100,\,\$0\}}\right) < \frac1n,
\]
i.e., whenever $1-v_{\{\$100,\,\$0\}}<1/(n-1)$.
Hence,  $L_b$ is increasingly preferred to $L_r$ when $n$ increases, which is consistent with the findings reported in \cite{JirousekShenoy2017}. In our model, when
\[
u_{\{\$100,\,\$0\}} < \frac1{n-1} < 1-v_{\{\$100,\,\$0\}},
\]
the two lotteries $L_r$ and $L_b$ are incomparable. If forced to choose, the DM might just choose arbitrarily. As the experiment reported in \cite{JirousekShenoy2017} did not allow the respondents to express inability to choose between the two lotteries, it does not provide any evidence for or against our model.
\end{example}

\begin{figure}[htbp]
\begin{center}
\includegraphics[width=4in]{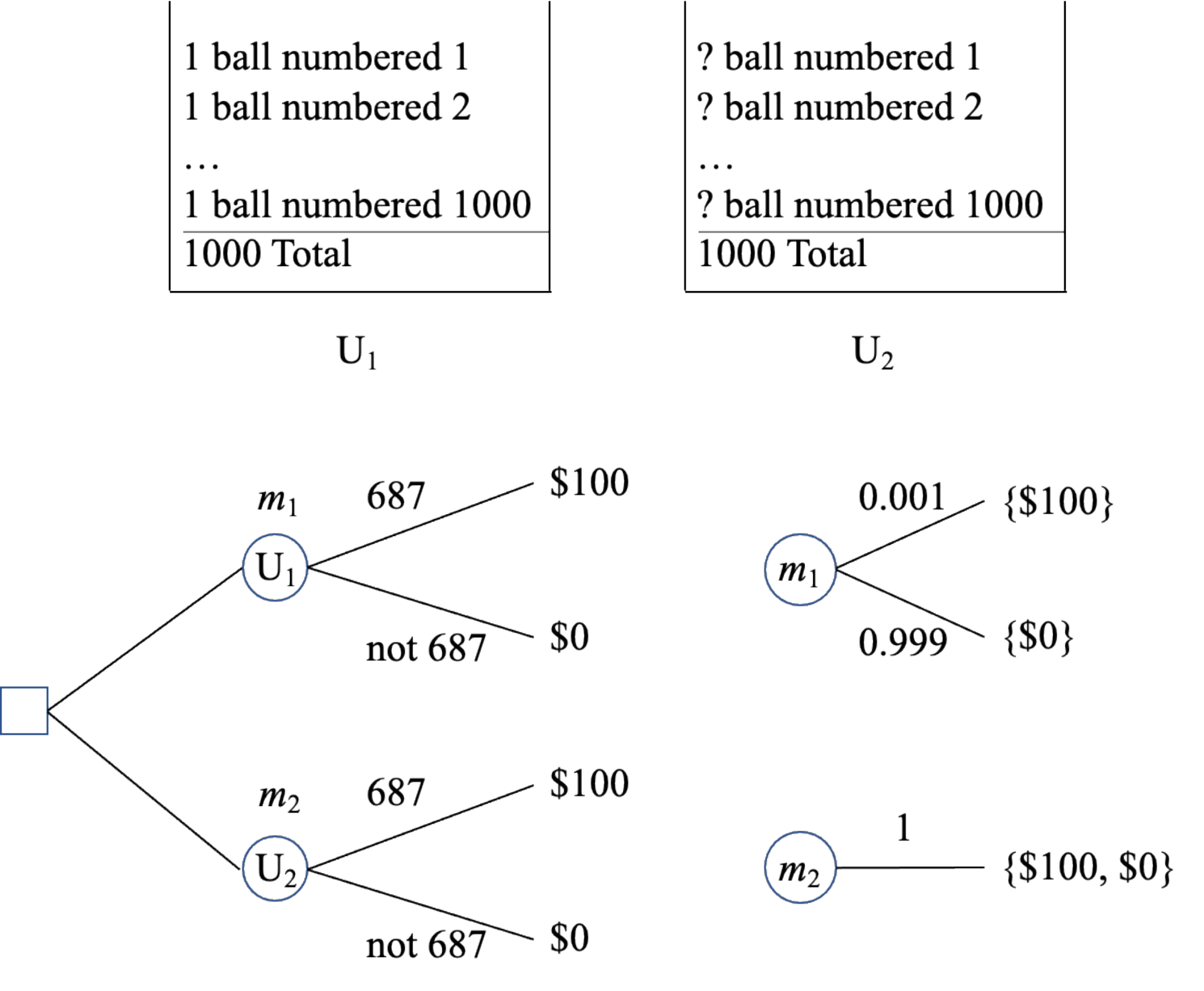}
\caption{Two urns with 1,000 balls, choices, and the corresponding belief function lotteries}
\label{fig:urnsw1000balls}
\end{center}
\end{figure}

\begin{example}[Urns with 1,000 balls]
\label{ex:1000balls}
The following example is discussed in \cite{BeckerBrownson1964}, where it is credited to Ellsberg in an oral conversation 
$($with the authors of \cite{BeckerBrownson1964}). It is also discussed 
in \cite{EinhornHogarth1986}. There are two urns, each with 1,000 balls, numbered from $1-1,000$. 
Urn $1$ has exactly one ball for each number, and there is no ambiguity. 
Urn $2$ has unknown number of balls of each number, and there is much ambiguity. One ball is to be chosen at random 
from an urn of your choosing. If the number on the ball matches a specific number, e.g., $687$, you win $\$100$, and if not, you win nothing $(\$0)$. 
Which one of the two urns will you choose? This choice problem is shown in Figure \ref{fig:urnsw1000balls}.

It is reported in \cite{BeckerBrownson1964} that many respondents chose Urn $2$. Why? 
Urn $1$ has only one ball numbered $687$, and therefore, the 
probability of winning $\$100$ if the choice is Urn $1$ is very small, $0.001$. Urn $2$ could possibly have anywhere from $0$ to $1,000$ balls 
numbered $687$. Thus, the choice of Urn $2$, although ambiguous, is appealing. Let's analyze this problem using Theorem $\ref{th:dsutility1}$.

Let $X_1$ denote the number on the ball chosen Urn $1$, and let $X_2$ denote the number on the ball chosen Urn $2$. 
$\Omega_{X_1} = \Omega_{X_2} = \{1, \ldots, 1000\}$. Function $m_{X_1}$ is a BPA for $X_1$ as follows: 
$m_{X_1}(\{1\}) = \ldots = m_{X_1}(\{1000\}) = 0.001$.
BPA $m_{X_2}$ is vacuous, i.e., $m_{X_2}(\Omega_{X_2}) = 1$.

Lottery $L_1$ corresponding to choice of Urn $1$ $($say, alternative $f_1)$ is $[\{\$100, \$0\}, m_1]$, where $m_1$ is a BPA for 
$\{\$100, \$0\}$ such that $m_1(\{\$100\}) = 0.001$, and $m_1(\{\$0\}) = 0.999$. $L_1$ is a bf reference lottery, and thus, $[u](L_1) = [0.001,0.001]$.
Lottery $L_2$ corresponding to choice of Urn $2$ $($say, alternative $f_2)$ is $[\{\$100, \$0\}, m_2 ]$, where 
$m_2$ is a vacuous BPA for $\{\$100, \$0\}$. The utility of $L_2$ is 
\[
[u](L_2) = \left[u_{\{\$100,\,\$0\}},1-v_{\{\$100,\,\$0\}}\right].
\] 
Consequently, $L_2$ is  preferred to $L_1$ as long as
\[
u_{\{\$100,\,\$0\}} \ge 0.001,
\]
a condition that is easily satisfied. This may explain why many DMs  appear  to be ambiguity-seeking in this context, i.e., prefer $L_2$ to $L_1$. 
\end{example}


\subsection{An additional assumption and the corresponding representation theorem}
\label{subsec:add_assump}
Whereas Theorem \ref{th:dsutility1} guarantees the existence of an interval valued utility function for bf lotteries, there remains the problem of practical elicitation of utilities. The maximum number of utilities to be elicited increases exponentially with the number $r$ of outcomes.  As the utilities of the worst and the best outcomes are, by construction, 0 and 1, and $w_{\{O_i\}}=0$ for each other single outcome $O_i$, the actual number of utilities to be elicited is, at most,  $2(2^r-1-r)+r-2=2^{r+1}-r-4$. By making one more reasonable assumption, we can drastically limit the number of parameters to be elicited.

Assumption \ref{a7b} below has no counterpart in the vN-M theory, but it is rooted in decision-making under ignorance \cite{Jaffray1989,Wakker2000}. For any nonempty subset of consequences $\sfa \subseteq \bO$, let $\underline{O}_{\sfa}$ and $\overline{O}_{\sfa}$ denote, respectively, the worst and the best outcome in $\sfa$. To simplify the notations, we assimilate a deterministic BPA with its focal set, and we write $\sfa \succsim \sfb$ to mean that the deterministic lottery with focal set $\sfa$ is preferred or indifferent to the deterministic lottery with focal set $\sfb$. Then our last assumption can be stated as follows.

\begin{assumption}[Dominance]
\label{a7b}
For all nonempty subsets $\sfa$ and $\textsf{b}$ of \bO, if 
$\underline{O}_\sfa \succsim \underline{O}_\sfb$ and $\overline{O}_{\sfa} \succsim \overline{O}_{\textsf{b}}$, 
then $\sfa \succsim \textsf{b}$.
\end{assumption}

Assumption \ref{a7b} implies that the preference between two deterministic lotteries with focal sets $\sfa$ and $\textsf{b}$ is determined only by the best and worst outcomes in $\sfa$ and $\textsf{b}$. In particular, when $\underline{O}_{\sfa}=\underline{O}_{\textsf{b}}$ and $\overline{O}_\sfa = \overline{O}_\sfb$, then $\sfa \sim \textsf{b}$. Although counterintuitive at first glance, this assumption cannot be avoided if we accept the PCI, i.e., if we accept that any two acts yielding the same bf lottery are equivalent, as shown by the following example. 

\begin{example}
\label{ex:dominance}
Assume that the set of outcomes is $\bO=\{\$0,\$1,\$999,\$1,000\}$. According to Assumption \ref{a7b}, a DM would be indifferent between receiving one of the prizes in $\sfa=\{\$0,\$1,\$1,000\}$, without any further information, and receiving one of the prizes in $\textsf{b}=\{\$0,\$999,\$1,000\}$. It may be argued that most DMs would strictly prefer $\sfb$ to $\sfa$. Yet, assume that the state space is $\Omega_X=\{\omega_1,\omega_2,\omega_3,\omega_4\}$, we are in a state of complete ignorance, i.e., $m_X(\Omega_X)=1$, and the deterministic lotteries $\sfa$ and $\sfb$ are generated by the acts $f_1$ and $f_2$ shown in Table \ref{tab:ex_dominance}. It is clear that $f_1$ dominates $f_2$ $($it yields at least as desirable consequences for all states of nature, and strictly preferred consequences for some states of nature$)$, so it would be paradoxical to strictly prefer $f_2$ over $f_1$, i.e., to strictly prefer $\sfb$ over $\sfa$. But $\sfa$ and $\sfb$ might also have been generated by acts $f'_1$ and $f'_2$ in Table \ref{tab:ex_dominance} and, as $f'_2$  dominates $f'_1$, it would also be paradoxical to strictly prefer $f'_1$ over $f'_2$, i.e., to strictly prefer $\sfa$ over $\sfb$. Consequently,  indifference between $\sfa$ and $\sfb$ seems to be the only rational option in this case.

\begin{table}
\caption{Acts in Example \ref{ex:dominance} \label{tab:ex_dominance}}
\begin{center}
\begin{tabular}{ccccc}
\hline
 & $\omega_1$ & $\omega_2$ & $\omega_3$ &  $\omega_4$\\
 \hline
$f_1$ & 0 & 1 & 1000 & 1000\\
$f_2$ & 0 & 0 & 999 & 1000\\
$f'_1$ & 0 & 1 & 1 & 1000\\
$f'_2$ & 0 & 999 & 1000 & 1000\\
\hline
\end{tabular}
\end{center}
\end{table}
\end{example}

Generalizing Example \ref{ex:dominance}, Jaffray \cite{Jaffray1989} shows that, whenever $\underline{O}_\sfa \succsim \underline{O}_\sfb$ and $\overline{O}_{\sfa} \succsim \overline{O}_{\textsf{b}}$, we can always construct a state space $\Omega_X$ and two acts $f_1$ and $f_2$ such that $f_1[\Omega_X]=\sfa$, $f_2[\Omega_X]=\sfb$ and, for any $\omega\in\Omega_X$, $f_1(\omega) \succsim f_2(\omega)$. As $f_1$ yields at least as desirable outcomes as $f_2$ under any state of nature, it should be preferred whatever our beliefs on $\Omega_X$, and in particular when $m_X(\Omega_X)=1$. Hence, we should have $\sfa \succsim \sfb$.

Assumption \ref{a7b} implies  that $\sfa \succsim \underline{O}_{\sfa}$ and  $ \overline{O}_{\sfa} \succsim \sfa$. From Assumption \ref{a6b}, we thus have $u_\sfa \ge u_{\underline{O}_{\sfa}}$ and $1-v_\sfa  \le u_{\overline{O}_{\sfa}}$. Consequently,  the utility bounds $u_{\sfa}$ and $1-v_\sfa$ of subset $\sfa$ can be written as convex combinations of the utilities of its worst and best outcomes:
\begin{subequations}
\label{eq:jaffray2}
\begin{align}
u_{\sfa} &= \alpha(\underline{O}_{\sfa},\overline{O}_{\sfa}) \, u_{\underline{O}_{\sfa}} + 
\left(1-\alpha(\underline{O}_{\sfa},\overline{O}_{\sfa})\right) \, u_{\overline{O}_{\sfa}}\\
1-v_{\sfa} &= \beta(\underline{O}_{\sfa},\overline{O}_{\sfa}) \, u_{\underline{O}_{\sfa}} + 
\left(1-\beta(\underline{O}_{\sfa},\overline{O}_{\sfa})\right) \, u_{\overline{O}_{\sfa}},
\end{align}
\end{subequations}
where $\alpha(\underline{O}_\sfa,\overline{O}_\sfa)$ and $\beta(\underline{O}_\sfa,\overline{O}_\sfa) $ are two coefficients depending only on the best and worst 
outcomes in $\sfa$, such that $0\le \alpha(\underline{O}_\sfa,\overline{O}_\sfa) \le \beta(\underline{O}_\sfa,\overline{O}_\sfa)\le 1$. In Jaffray's framework \cite{Jaffray1989}, $w_\sfa=0$ (see Section \ref{subsec:Jaffray}) and $\alpha(\underline{O}_\sfa,\overline{O}_\sfa)$ is called a \emph{local pessimism index}. In our framework, we can see the interval  $\left[\alpha(\underline{O}_\sfa,\overline{O}_\sfa),\beta(\underline{O}_\sfa,\overline{O}_\sfa)\right]$ as an interval-valued local pessimism index reflecting both the DM's attitude to ambiguity and indeterminacy. Assumption \ref{a7b} thus brings the maximum number of parameters to be elicited from $2^{r+1}-r-4$ down to $r(r-1)+r-2=r^2-2$. The above discussion can be summarized in the form of the following representation theorem (generalizing Theorem 2 in \cite{Jaffray1989}).

\begin{theorem}[Interval-valued local pessimism index]
\label{th:dsutility3}
Suppose $L = [\bO, m]$ and $L^\prime = [\bO, m^\prime]$ are bf lotteries on $\bO$. 
If the preference relation $\succsim$ on $\mathcal{L}_{bf}$ satisfies Assumptions \ref{a1b}--\ref{a7b}, then there are numbers $u_{O}$ associated with outcomes $O \in \bO$ and two mappings $\alpha$ and $\beta$ from
\[
{\cal O}=\{(O,O')\in \bO^2: O'\succsim O\}
\] 
to $[0,1]$, with $\alpha\le\beta$, such that  $L \succsim L'$ if and only if
\begin{multline*}
\sum_{\emptyset \neq \sfa\subseteq \bO} m(\sfa)\,\left[\alpha(\underline{O}_{\sfa},\overline{O}_{\sfa}) u_{\underline{O}_{\sfa}} + (1- \alpha(\underline{O}_{\sfa},\overline{O}_{\sfa})) u_{\overline{O}_{\sfa}}\right] \ge \\
\sum_{\emptyset \neq \sfa\subseteq \bO} m^\prime(\sfa)\,\left[\alpha(\underline{O}_{\sfa},\overline{O}_{\sfa}) u_{\underline{O}_{\sfa}} + (1- \alpha(\underline{O}_{\sfa},\overline{O}_{\sfa})) u_{\overline{O}_{\sfa}}\right]
\end{multline*}
and
\begin{multline*}
\sum_{\emptyset \neq \sfa\subseteq \bO} m(\sfa)\,\left[\beta(\underline{O}_{\sfa},\overline{O}_{\sfa}) u_{\underline{O}_{\sfa}} + (1- \beta(\underline{O}_{\sfa},\overline{O}_{\sfa})) u_{\overline{O}_{\sfa}}\right] \ge \\
\sum_{\emptyset \neq \sfa\subseteq \bO} m^\prime(\sfa)\,\left[\beta(\underline{O}_{\sfa},\overline{O}_{\sfa}) u_{\underline{O}_{\sfa}} + (1- \beta(\underline{O}_{\sfa},\overline{O}_{\sfa})) u_{\overline{O}_{\sfa}}\right],\end{multline*}
where $\underline{O}_{\sfa}$ and $\overline{O}_{\sfa}$ are, respectively, the worst and the best outcomes in $\sfa\subseteq \bO$. Thus, for a bf lottery $L = [\bO, m]$, we can define 
\begin{equation*}
[u](L) = [u, 1-v]
\end{equation*}
as an interval-valued utility of $L$, with
\begin{equation*}
u=\sum_{\emptyset \neq \sfa\subseteq \bO} m(\sfa)\,\left[\alpha(\underline{O}_{\sfa},\overline{O}_{\sfa}) u_{\underline{O}_{\sfa}} + (1- \alpha(\underline{O}_{\sfa},\overline{O}_{\sfa})) u_{\overline{O}_{\sfa}}\right] 
\end{equation*}
and
\[
 1-v=\sum_{\emptyset \neq \sfa\subseteq \bO} m(\sfa)\,\left[\beta(\underline{O}_{\sfa},\overline{O}_{\sfa}) u_{\underline{O}_{\sfa}} + (1- \beta(\underline{O}_{\sfa},\overline{O}_{\sfa})) u_{\overline{O}_{\sfa}}\right].
\]

 Also, this utility function is unique up to a strictly increasing affine transformation. 
\end{theorem}

\subsection{A simpler model}
\label{subsec:simplermodel}

Strat \cite{Strat1990} proposes independently, but without any axiomatic justification, a  criterion similar to that of Theorem \ref{th:dsutility3} but with a constant parameter $\alpha(\underline{O}_{\sfa},\overline{O}_{\sfa})=\beta(\underline{O}_{\sfa},\overline{O}_{\sfa})=\alpha$ that does not depend on the subset $\sfa$.  In a similar way, we can assume that the lower and upper pessimism indices take on constant values: $\alpha(\underline{O}_{\sfa},\overline{O}_{\sfa})=\alpha$ and $\beta(\underline{O}_{\sfa},\overline{O}_{\sfa})=\beta$, with $0\le\alpha\le\beta\le 1$. This simple model depends on only $r$ parameters: the utilities of the single outcomes $u_{\{O_i\}}$ for $i=2,\ldots,r-1$ and the two coefficients $\alpha$ and $\beta$. It allows us to recover some existing decision criteria as special cases:
\begin{itemize}
\item When $\alpha=\beta$, the utility interval $[u](L)$ is reduced to a point $u(L)$ and we get Strat's criterion, also called the generalized Hurwicz criterion in \cite{Denoeux2019}, which is a special case of the real-valued utility \eqref{eq:dsutility} in Corollary \ref{th:dsutility2};

\item In particular, when $\alpha=\beta=0$, then 
\begin{equation}
\label{eq:lower}
u(L)= \sum_{\sfa \subseteq \bO} m(\sfa) \;  \min_{O\in \sfa} u_O,
\end{equation}
which is the lower expected utility $\underline{u}_m$ with respect to $m$ \cite{Dempster1967,Shafer1981,DempsterKong1987}. As shown by Gilboa and Schmeidler \cite{GilboaSchmeidler1994}, $\underline{u}_m$ is also the Choquet expected utility \cite{Choquet1953} with respect to the belief function $Bel_m$ corresponding to $m$. The preference relation between bf lotteries then corresponds to the maximin criterion, which reflects a  pessimistic attitude of the DM.

\item Similarly, when $\alpha=\beta=1$, we get 
\begin{equation}
\label{eq:upper}
u(L)= \sum_{\sfa \subseteq \bO} m(\sfa) \;  \max_{O\in \sfa} u_O,
\end{equation}
which is the upper expected utility $\overline{u}_m$, or the Choquet expected utility with respect to the plausibility function $Pl_m$ corresponding to $m$ \cite{GilboaSchmeidler1994}. The corresponding decision strategy corresponds to the maximax criterion, which models an optimistic attitude of the DM.

\item When $\alpha=0$ and $\beta=1$, then the interval-valued utility is equal to the lower-upper expected utility interval
\[
[u](L)= \left[\underline{u}_m, \overline{u}_m\right].
\]
The corresponding preference relation is then the \emph{interval bound dominance} relation \cite{Destercke2010,Denoeux2019}, defined by
\begin{equation}
\label{eq:interv_dom}
L \succsim L' \Leftrightarrow \left(\underline{u}_m \ge \underline{u}_{m'} \quad \text{and} \quad \overline{u}_m \ge \overline{u}_{m'}\right).
\end{equation}
\end{itemize}
 
 In the general case, we have
 \begin{equation}
 \underline{u}_m \le u \le 1-v \le  \overline{u}_m,
 \end{equation}
where $u$ and $w$ are as in Eq. \eqref{eq:uw}. Thus, the interval-valued utility $[u](L)$ of lottery $[\bO, m]$ as defined in 
Theorem \ref{th:dsutility1} is always included in the lower-upper expected utility interval, and the preference relation induced by our interval-valued utilities compares more bf lotteries than the interval dominance relation \eqref{eq:interv_dom}. The lower and upper expectations defined by Eqs. \eqref{eq:lower}-\eqref{eq:upper} can thus be seen as lower and upper bounds of the interval utility of a lottery $L=[\bO,m]$ and could be used as conservative estimates if parameters $\alpha$ and $\beta$ cannot be elicited.

\begin{example}
Assume that the set of outcome is $\bO=\{\$0, \$10, \$50, \$100\}$. The full model (without Assumption \ref{a7b}) requires the assessment of 24 parameters: $u_{\{\$10\}}$, $u_{\{\$50\}}$, and the $u$ and $v$ values for the 11 subsets of $\bO$ with cardinality strictly greater than 1. With Assumption \ref{a7b}, the number of parameters to be elicited is down to 14: $u_{\{\$10\}}$, $u_{\{\$50\}}$, and the $\alpha$ and $\beta$ values for the following pairs of worst and best outcomes: $(\$0,\$10)$, $(\$0,\$50)$, $(\$0,\$100)$, $(\$10,\$50)$, $(\$10,\$100)$ and $(\$50,\$100)$. Assuming $\alpha$ and $\beta$ to be constant brings the number of parameters to only 4.
\end{example}

\paragraph{A practical elicitation procedure} Whatever the simplifying assumptions made, the trickiest part for eliciting the interval-valued utility of a bf lottery resides in the determination of the equivalent bf reference lottery for any non-singleton focal set $\sfa$ (if $\alpha$ and $\beta$ are assumed to be constant,  this determination needs to be done for only one non-singleton focal set). Let $\tsfa=[\bO_2, (u_\sfa,v_\sfa,w_\sfa)]$ be the bf reference lottery equivalent to $\sfa$ (assumed to exist from Assumption \ref{a3b}). For any a probabilistic reference lottery $L=[\bO_2,(u, 1-u)]$, there are three cases:
\begin{enumerate}
\item If $u \ge 1-v_\sfa$, then $L \succ \tsfa$;
\item If $u \le u_\sfa$, then $ \tsfa \succ L$;
\item If $u_\sfa < u \le 1-v_\sfa$ or $u_\sfa \le u < 1-v_\sfa$, then $\tsfa$ and $L$ are incomparable.
\end{enumerate}
To determine $u_\sfa$ and $v_\sfa$, we can thus start with $u=0$ and gradually increase $u$ until $\tsfa$ and $L$ become incomparable, which gives us $u_\sfa$, and then gradually decrease $u$ from $u=1$ until  $\tsfa$ and $L$ become incomparable, which gives us $v_\sfa$. Parameters $\alpha$ and $\beta$ are then obtained by solving Eqs. \eqref{eq:jaffray2}. This procedure was used implicitly in Example \ref{ex:ex3}.

\section{Comparison with Some Existing Decision Theories}
\label{sec:comparison}
In this section, we compare our utility theory to Jaffray's linear utility theory \cite{Jaffray1989}, Smets' two-level decision theory \cite{Smets2002}, decision theories for possibility theory \cite{Duboisetal1999,GiangShenoy2005} and partially consonant belief functions \cite{GiangShenoy2011},  and Shafer's constructive decision theory \cite{Shafer2016}.


\subsection{Comparison with Jaffray's axiomatic theory}
\label{subsec:Jaffray}
Jaffray's axiomatic theory is based on considering the set of all belief functions on $\bO$ as a mixture set as follows. 
Suppose $m_1$ and $m_2$ are BPAs for $\bO$, and suppose $\lambda \in [0, 1]$. Then $m$ defined as:
\begin{equation}
\label{eq:mixture}
m(\sfa) = \lambda\,m_1(\sfa) + (1 - \lambda) m_2(\sfa)
\end{equation}
for all $\sfa \in 2^{\bO}$, is a BPA for $\bO$. BPA $m$ can be written as $m = \lambda\,m_1 + (1-\lambda)m_2$, and called a \emph{mixture} of $m_1$ and $m_2$. Using the Jensen-version \cite{Jensen1967} of vN-M axiom system, Jaffray uses the following axioms, all of which are expressed using mixture BPA functions:

\begin{assumption}[Completeness and transitivity]
\label{a1j}
The relation $\succsim$ is complete and transitive over $\mathcal{L}_{bf}$.
\end{assumption}

\begin{assumption}[Independence]
\label{a2j}
For all $L_1 = [\bO, m_1]$ and $L_2 = [\bO, m_2]$ in $\mathcal{L}_{bf}$, 
and $\lambda \in (0, 1)$, 
$L_1 \succ L_2$ implies $[\bO, \lambda\,m_1 + (1-\lambda)\,m] \succ [\bO, \lambda\,m_2 + (1-\lambda)\,m]$.
\end{assumption}

\begin{assumption}[Continuity]
\label{a3j}
For all $L_1 = [\bO, m_1]$, $L_2 = [\bO, m_2]$, and $L_3 = [\bO, m_3]$ in 
$\mathcal{L}_{bf}$ 
such that 
$L_1 \succ L_2 \succ L_3$, there exists $\lambda$ and $\mu$ in $(0, 1)$ such that 
\[
[\bO, \lambda\,m_1 + (1-\lambda)\,m_3] \succ [\bO, m_2] \succ [\bO, \mu\,m_1 + (1-\mu)\,m_3].
\]
\end{assumption}

\begin{theorem}[Jaffray's representation theorem \cite{Jaffray1989}]
The preference relation $\succsim$ on $\mathcal{L}_{bf}$ satisfies Assumptions \ref{a1j}--\ref{a3j} if and only if there exists a utility function $u:\mathcal{L}_{bf} \to \mathbb{R}$ such that for any lottery $L = [\bO, m]$ in $\mathcal{L}_{bf}$,
\begin{equation}
\label{eq:uJaffray}
u(L) = \sum_{\emptyset \neq \sfa \subseteq \bO} m(\sfa)\,u_{\sfa}
\end{equation}
where $u_{\sfa} = u([\bO, m_{\sfa}^d])$, and $m_{\sfa}^d$ is a deterministic BPA for $\bO$ such that $m_{\sfa}^d(\sfa) = 1$.
\end{theorem}

Thus, Jaffray's axioms result in the same solution as that of Corollary \ref{th:dsutility2}, which is a special case of Theorem \ref{th:dsutility1}. As Jaffray's axioms do not use Dempster's rule explicitly, it is not clear whether Eq. \eqref{eq:uJaffray} applies to the D-S framework or not. The mixture BPA $m$ derived from BPAs $m_1$ and $m_2$ using Eq. \eqref{eq:mixture} is not Dempster's combination rule, although Eq. \eqref{eq:mixture} can be derived from a belief function model using Dempster's rule. In \cite{Jaffray1994}, Jaffray writes:
\begin{quote}``It has been shown by \cite{Jaffray1989, Jaffray1991} that, in the \emph{lower probability interpretation} of belief functions, the axioms of von Neumann-Morgenstern linear utility theory could be justified with the same arguments as in the case of risk (probabilized uncertainty)'' (emphasis added).
\end{quote} 
Also, in \cite{JaffrayWakker1993}, Jaffray and Wakker write: 
\begin{quote}
``Given the widespread use of belief functions, it is remarkable that only recently were decision criterion for the above-mentioned type of situations proposed and axiomatized in \cite{Jaffray1989}. He uses as a primitive axiom the independence condition with respect to mixtures of belief functions over the outcomes to generalize expected utility: \cite{Jaffray1991} justifies this condition by means of a \emph{lower-probability interpretation} of belief functions.'' (emphasis added).
\end{quote}

Thus, it is clear that Jaffray has in mind the credal set semantics of belief functions, which are inconsistent with Dempster's combination rule\footnote{It is possible that in 1989, it was not well understood that credal set semantics of belief functions were incompatible with Dempster's combination rule. This was apparently clarified in the early 1990s by Shafer in \cite{Shafer1990, Shafer1992} and also by Fagin and Halpern in \cite{FaginHalpern1991, HalpernFagin1992}}.

In comparison, our representation theorem is based on a set of axioms making use of the basic constructs of DS theory (namely, Dempster's combination rule, marginalization, and conditional embedding), we provide more compelling arguments supporting Eq. \eqref{eq:uJaffray} as a natural definition of the real-valued utility of a bf lottery in the D-S theory.

Also, there is no explicit notion of a bf reference lottery in Jaffray's framework. Thanks to our continuity axiom (Assumption \ref{a3b}), the interval-valued utility $[u_{\sfa}, 1-v_{\sfa} ]$ in our framework receives a simple interpretation as an interval-valued probability of a best outcome $O_1$, in a bf reference lottery $[\bO_2, m_\sfa]$ that is indifferent to $\sfa$ and such that $m_\sfa(\{O_1\}) = u_\sfa$, $m_\sfa(\{O_r\}) = v_\sfa$, and $m_\sfa(\bO_2) = 1-u_\sfa-v_\sfa$. We believe that this simple interpretation can be very helpful when eliciting utilities from DMs, as discussed in Section \ref{subsec:simplermodel}. 


\subsection{Comparison with Smets' decision theory}
\label{subsec:pignistic}
Smets' decision theory \cite{Smets2002} is a two-level framework where beliefs, represented by belief functions, are held at a credal level. 
When a DM has to make a decision, the marginal belief function for a variable of interest is transformed into a PMF, and the Bayesian 
expected utility framework is then used to make a decision. 

Smets uses a transformation called the \emph{pignistic} transform to transform belief functions into PMFs. 
This transform is justified in \cite{Smets2005} using a mixture property as follows. Let $T$ denote the belief-PMF transformation.  
Smets \cite{Smets2005} argues that this transformation should be linear, i.e., we should have, for any $\lambda \in [0,1]$, 
\begin{equation}
\label{eq:mixtureconsistency}
T(\lambda\,m_1 + (1-\lambda)\,m_2) = \lambda T(m_1) + (1-\lambda)T(m_2).
\end{equation}
The unique transformation $T$ verifying  (\ref{eq:mixtureconsistency}) is the pignistic transformation defined as $T(m)=BetP_m$ with
\begin{equation}
\label{eq:betp}
BetP_m(O)=\sum_{\sfa \subseteq \bO} \frac{m(\sfa)}{|\sfa|} \, I(O \in \sfa)
\end{equation}
for all $O\in\bO$. The pignistic PMF $BetP_m$ is mathematically identical to the  
Shapley value in cooperative game theory \cite{shapley53}. In \cite{Smets2005}, Smets attempts to derive Eq. \eqref{eq:mixtureconsistency} 
from the maximum expected utility principle. The argument, however, is quite  technical and not very compelling.

Given the definition in Eq. \eqref{eq:betp}, the expected utility of a bf lottery $L=[\bO,m]$ according to the pignistic PMF is
\begin{subequations}
\label{eq:pignistic}
\begin{align}
u_{BetP}(L)&=\sum_{O \in \bO} BetP_m(O) \, u_{\{O\}}\\
&=\sum_{O \in \bO} \left(\sum_{\sfa \subseteq \bO} \frac{m(\sfa)}{|\sfa|} \, I(O \in \sfa)\right) u_{\{O\}}\\
&=\sum_{\sfa \subseteq \bO}  m(\sfa) \left(\frac{1}{|\sfa|} \sum_{O \in \sfa} u_{\{O\}}\right).
\end{align}
\end{subequations}
It is a special case of Eq. \eqref{eq:dsutility}, with
\[
u_{\sfa}=\frac{1}{|\sfa|} \sum_{O \in \sfa} u_{\{O\}}.
\]
Smets' decision theory   thus  amounts to assuming that a DM is indifferent between a bf lottery that gives them an outcome in $\sfa$ for sure, 
and a bf reference lottery in which the probability of the best outcome is equal to the average utilities of the outcomes in $\sfa$. This is consistent with our Assumptions \ref{a1b}--\ref{a6b}, but it is inconsistent with Assumption \ref{a7b}. For instance, in Example \ref{ex:dominance}, a DM using the pignistic criterion would  strictly prefer $\sfb$ to $\sfa$, even though the act $f_1$ generating $\sfa$ dominates the act $f_2$ generating $\sfb$. Moreover, this restricted model does not have any parameter to represent a DM's attitude toward ambiguity. As a result, it is unable to explain Ellsberg's paradox and the ambiguity aversion of human DMs as described, e.g., in the examples presented in Section \ref{subsec:theorems}. 

\subsection{Comparison with other axiomatic theories}
In this subsection, we compare our axiomatic decision theory with other axiomatic decision theories for D-S belief functions.

Dubois et al. \cite{Duboisetal1999} describe an axiomatic decision theory for the case where uncertainty is described by a possibility distribution, which is a special case of a belief function with nested focal elements (such a belief function is said to be consonant). Dubois et al.'s decision theory consists of two sets of axioms, one for the pessimistic case, and one for the optimistic case. In contrast, Giang and Shenoy \cite{GiangShenoy2005} propose an axiomatic theory for possibility theory with one set of axioms, and the utility function is binary-valued (binary-valued utilities are possibility distribution values of $O_1$ and $O_r$ for possibilistic reference lotteries). The two axiomatic theories for possibility theory are compared in detail in \cite{GiangShenoy2005}. The latter theory is generalized in terms of partially consonant belief functions in \cite{GiangShenoy2011}. A partially consonant belief function is a belief function where the set of focal elements can divided into groups such that (a) the focal elements in different groups are disjoint, and (b) the focal elements in the same group are nested. The family of partially consonant belief functions include Bayesian belief functions and consonant belief functions.

Giang \cite{Giang2012} compares the Giang-Shenoy decision theory for partially consonant belief functions with Jaffray's axiomatic decision theory for general belief functions. Similar to Jaffray's theory, our decision theory is for the case of general belief functions. While our utility is interval-based, leading to incomplete preferences, Jaffray's theory for general belief functions, and Giang-Shenoy's theory for partially consonant belief functions based on binary utility, result in complete preferences, which is a special case of our theory. Walley \cite{Walley1987} argues that partially consonant belief functions is the only class of D-S belief functions that is consistent with the likelihood principle of statistics, but this argument applies only to statistical inference, and not to uncertain reasoning in general.


\subsection{Comparison with Shafer's constructive decision theory}
Shafer \cite{Shafer2016} argues for a decision theory that allows us to construct both goals and beliefs in response to a decision. In the vN-M utility theory, 
we start with a probabilistic lottery, and construct a utility function that reflects a DM's risk attitude. Thus, probabilities and utilities are separate constructs 
that are then combined for the computation of expected utility. In many situations, we have neither objective nor subjective probabilities. For such 
situations, Shafer argues for constructing belief functions from available evidence, and constructing a set of consistent and monotonic goals.
Given a set of actions, we examine which goals each of the actions will achieve. We use belief functions to make judgments based on evidence about 
what will happen if an action is taken. We then use these belief functions to compute the expected number of goals that an action will satisfy, and 
pick an action that satisfies the most goals. This can be generalized to the case where not all goals are equally weighted, some are weighted more than 
others.

Our utility theory is more in line with vN-M utility theory than Shafer's constructive decision theory. There is considerable literature in many domains about 
the use of utility theory for decision making. While Shafer's constructive decision theory is intriguing and may indicate an interesting direction to explore, there is much to be done before 
we can apply it in many domains for which we have a decision theory in the vN-M style.


\section{Summary and Conclusions}
\label{sec:sandc}
In this section, we summarize our proposal and sketch some future work. We start with Luce and Raiffa's version of the vN-M utility theory for probabilistic lotteries. We then consider bf lotteries, lotteries when our beliefs about the state of the world is described by DS belief functions. We use a similar set of axioms as vN-M, but first we replace each singleton outcome in a probabilistic lottery by a focal set of a BPA. Second, we  replace the reduction of compound lotteries with a corresponding axiom that uses Dempster's combination rule and belief function marginalization in place of probabilistic combination (pointwise multiplication followed by normalization) and probabilistic marginalization (addition). Third, we use a bf reference lottery with two independent parameters. The axioms lead to a decision theory that involves assessing the utility of each focal 
element of a BPA as an interval-valued utility. Interval-valued utilities lead to a partial preference relation on the set $\mathcal{L}_{bf}$ of all bf lotteries. If we use Bayesian bf reference lotteries with a single parameter, then our axiomatic framework leads to a real-valued utility function that is exactly the same as in Jaffray's linear utility theory \cite{Jaffray1989}.

The decision theory that results from our axioms is more general than that proposed by Jaffray \cite{Jaffray1989}, which can be construed as a decision theory for belief functions interpreted as generalized probabilities. Jaffray's axiomatic theory is based on a set of mixture BPAs. A mixture of two BPAs is not the same as a Dempster's combination of two BPAs, although we could construct a belief function model where the mixture BPA is obtained by Dempster's rule. Thus, it is not clear if Jaffray's linear utility theory is applicable to D-S belief function lotteries or not. Our utility theory confirms that this is indeed the case. Our bf reference lotteries lead to interval-valued utilities, and consequently, a partial preference relation on the set of all bf lotteries. 

We also compare our axiomatic theory to Smets' two-level framework \cite{Smets2002, Smets2005}, and note that his framework is too constrained to explain ambiguity-aversion or ambiguity-seeking behavior of human DMs. Other axiomatic decision theories proposed by Dubois et al. \cite{Duboisetal1999} and Giang and  Shenoy \cite{GiangShenoy2005,GiangShenoy2011} are restricted to consonant or quasi-consonant belief functions. Shafer \cite{Shafer2016} has recently published his constructive decision theory where he rejects the separation of beliefs and utilities. He proposes, instead, constructing a set of consistent and monotonic goals, and measuring the utility of each choice by the number of goals (or weighted goals) achieved by the choice. Shafer's constructive decision theory needs to be fleshed out before it can be applied to practical decision-making situations.

In practice, implementing the most general form of our axiomatic theory may need assessment of $2\,k$ parameters, where $k$ is the number of focal sets of a bf lottery. In the worst case, $k$ can be as large as $2^{|\bO|}-1$. Based on additional assumptions, we propose a model based on only two parameters, which can be interpreted as reflecting both the DM's attitude to ambiguity and their indeterminacy. This model, as well as others, will have to be further studied and developed. More generally, a rigorous methodology to elicit interval-valued  utilities remains to be designed and validated experimentally. 

Finally, in this paper, we start from the assumption that the D-S formalism is an adequate model of an agent's state of knowledge, and we derive a corresponding decision theory from a set of rationality requirements. Thus, a belief function on the state space is assumed to be given, and we generate interval-valued expected utilities for bf lotteries. A further step would be to justify not only utilities, but also the D-S calculus itself (including belief functions and Dempster's rule)  from properties of the DM's preference relation over acts,  similar to what Savage \cite{savage51} did  to provide a foundation for decision-making with probability theory, similar to what Dubois et al. \cite{dubois01b} did to justify decision-making with qualitative possibility theory, and similar to what Gul and Pesendorfer \cite{GulPesendorfer2014}, and Zhou et al. \cite{Zhouetal2018} did for decision-making with a theory of belief functions where the belief functions are interpreted as credal sets. This task remains to be done.


\section*{Acknowledgments}
This research was done during a visit by the second author to Universit\'{e} de technologie de Compi\`{e}gne (UTC) in Spring 2019. This research was funded partially by a sabbatical from the University of Kansas, and supported by the Labex MS2T, which is funded by the French Government through the program ``Investments for the future'' by the National Agency for Research (reference ANR-11-IDEX-0004-02). Thanks to S\'{e}bastien Destercke for many discussions on the topic of this manuscript, and to Suzanna Emelio for proofreading this manuscript. A short version of this paper appeared as \cite{DenoeuxShenoy2019b}, and we are grateful to the reviewers of our manuscript submitted to ISIPTA-2019 for their comments and suggestions. Thanks also to the three reviewers of our original manuscript submitted to IJAR for their comments and questions, our revision has benefitted from their comments.

\bibliographystyle{abbrv}
\bibliography{references1}

\end{document}